\documentclass[11pt,a4paper]{article}
\usepackage[utf8]{inputenc}
\usepackage[T1]{fontenc}
\usepackage{amsmath, amsthm}
\usepackage{amsfonts}
\usepackage{amssymb}
\usepackage{lmodern}
\usepackage{upgreek}
\usepackage[left=2cm,right=2cm,top=2cm,bottom=2cm]{geometry}
\usepackage{dsfont}
\usepackage[dvipsnames]{xcolor} 
\usepackage{tikz}
\usepackage[hidelinks]{hyperref}
\usepackage[ruled,vlined]{algorithm2e}
\usepackage{datetime}
\usepackage{subfig}
\usepackage{sublabel}
\usepackage[normalem]{ulem} 
\usepackage{placeins} 

\newcommand{\Esp}[1]{\mathrm{E}\left[#1\right]}
\newcommand{\Var}[1]{\mathrm{V}\left[#1\right]}
\newcommand{\Cov}[1]{\mathrm{Cov}\left[#1\right]}
\newcommand{\norm}[1]{{\Vert{}#1\Vert}}
\newcommand{\abs}[1]{{\vert{}#1\vert}}
\newcommand{\transpose}[1]{{#1}^t}
\newcommand{\numax}{\bar\nu}
\newcommand{\CalgoTot}{\mathcal{C}}
\newcommand{\CalgoAlpha}{\mathcal{C}_\alpha}
\newcommand{\CalgoBeta}{\mathcal{C}_\beta}

\newcommand{\R}{\mathds{R}}

\newcommand{\set}[1]{{\lbrace{#1}\rbrace}}

\newcommand{\Indic}[1]{\mathbf{1}_{\set{#1}}}

\newtheorem{example}{Example}
\newtheorem{proposition}{Proposition}
\newtheorem{definition}{Definition}

\newcommand{\accolade}[1]{\left\{{\renewcommand{\arraystretch}{1.15}\begin{array}{lcl}#1\end{array}}\right.}

\newcommand{\accoladesplit}[1]{\left\{\begin{aligned}#1\end{aligned}\right.}

\newcommand{\virguleacc}{\rlap{\,,}}
\newcommand{\pointacc}{\rlap{\,.}}
\newcommand{\ie}{i.e.\@ }
\newcommand{\eg}{e.g.\@ }

\newcommand{\tcpx}[1]{\;}

\newcommand{\M}{M}
\newcommand{\Y}{Y}
\newcommand{\agg}{{\mathcal{A}}}
\newcommand{\full}{{full}}

\newcommand{\ns}{s}

\newcommand{\KM}{K_M(x)}
\newcommand{\kxM}{k_M(x)}
\newcommand{\Mpx}{{\M_i(x),i\in\agg}}

\newcommand{\complexity}[1]{\mathcal{C}_{\mathrm{#1}}}
\newcommand{\Ccov}{\complexity{cov}}

\newcommand{\boxand}{\quad \mbox{ and } \quad}
\newcommand{\setI}{{\agg}}

\SetCommentSty{mycommfont}

\makeatletter
\newcommand{\shorteq}{%
  \settowidth{\@tempdima}{a}
  \, \resizebox{\@tempdima}{\height}{=} \,%
}
\makeatother

\newcommand{\did}[1]{#1}

\title{Nested Kriging predictions for datasets with large number of observations}
\author{
Didier Rulli\`{e}re\footnote{Universit\'{e} de Lyon, Universit\'{e} Claude Bernard Lyon~1, ISFA, Laboratoire SAF, EA2429, 50 avenue Tony Garnier, 69366 Lyon, France, didier.rulliere@univ-lyon1.fr.},
Nicolas Durrande\footnote{\emph{Corresponding author}, Mines Saint-\'Etienne, H. Fayol Institute, 158 cours Fauriel, Saint-\'Etienne, France and CNRS LIMOS, UMR 5168,  durrande@emse.fr.},
Fran\c cois Bachoc\footnote{Institut de Math\'ematiques de Toulouse, Universit\'e Paul Sabatier, 118 route de Narbonne, 31062 Toulouse, France, francois.bachoc@math.univ-toulouse.fr.}  \ and
Cl\'ement Chevalier\footnote{Institute of Statistics, University of Neuch\^atel, avenue de Bellevaux 51, 2000 Neuch\^atel, Switzerland, clement.chevalier@unine.ch.}.}

\newcommand{\Id}{I_d}
\begin{document}

\maketitle

\begin{abstract}
This work falls within the context of predicting the value of a real function at some input locations given a limited number of observations of this function. The Kriging interpolation technique (or Gaussian process regression) is often considered to tackle such a problem but the method suffers from its computational burden when the number of observation points is large. We introduce in this article nested Kriging predictors which are constructed by aggregating sub-models based on subsets of observation points. This approach is proven to have better theoretical properties than other aggregation methods that can be found in the literature. \did{Contrarily} to some other methods it can be shown that the proposed aggregation method is consistent. Finally, the practical interest of the proposed method is illustrated on simulated datasets and on an industrial test case with $10^4$ observations in a 6-dimensional space.
\end{abstract}
\textbf{Keywords:} Gaussian process regression $\cdot$ big data $\cdot$ aggregation methods $\cdot$ best linear unbiased predictor $\cdot$ spatial processes.


\section{Introduction}

\paragraph{}
\begin{sloppypar}
Gaussian process regression models have proven to be of great interest in many fields when it comes to predict the output of a function $f : D \rightarrow \R$, $D\subset \R^d$,  based on the knowledge of $n$ input-output tuples $(x_i,f(x_i))$ for $1 \leq i \leq n $ \cite{stein2012interpolation,santner2013design,Rasmussen2006}. One asset of this method is to provide not only a mean predictor but also a quantification of the model uncertainty. The Gaussian process regression framework uses a (centered) real-valued Gaussian process $\Y$ over $D$ as a prior distribution for $f$ and approximates it by the conditional distribution of $Y$ given the observations $Y(x_i)=f(x_i)$ for $1 \leq i \leq n $.
In this framework, we denote by $k: D \times D \rightarrow \mathds{R}$ the covariance function (or kernel) of $\Y$: $k(x,x')=\Cov{\Y(x),\Y(x')}$, and by $X \in D^n$ the  vector of observation points with entries $x_i$ for $1 \le i \le n$.
\end{sloppypar}
\paragraph{}
In the following, we use classical vectorial notations: for any functions $f: D \rightarrow \R$, $g: D \times D \rightarrow \R$ and for any vectors $A=(a_1,\ldots, a_n) \in D^n$ and $B=(b_1, \ldots, b_m) \in D^m$, we denote by $f(A)$ the $n \times 1$ real valued vector with components $f(a_i)$ and by $g(A,B)$ the $n \times m$ real valued matrix with components $g(a_i, b_j)$, $i=1,\dots, n$, $j=1,\dots, m$. With such notations, the conditional distribution of $\Y$ given the $n \times 1$ vector of observations $\Y(X)$ is Gaussian with mean, covariance and variance:
\begin{equation}
\accoladesplit{
	\M_\full(x) &= \Esp{\Y(x)|\Y(X)} = k(x,X) k(X,X)^{-1} \Y(X) \virguleacc \\
	c_\full(x,x') &= \Cov{\Y(x),\Y(x')|\Y(X)} = k(x,x') - k(x,X) k(X,X)^{-1} k(X,x') \virguleacc\\
v_\full(x) &= c_\full(x,x) \pointacc
}
\label{eq:fullmodel}
\end{equation}
Since we do not specify yet the values taken by $\Y$ at $X$, the ``mean predictor'' $\M_\full(x)$ is random so it is denoted by an upper-case letter $M$. The approximation of $f(x)$ given the observations $f(X)$ is thus given by $m_\full(x) = \Esp{\Y(x)|\Y(X)=f(X)}= k(x,X) k(X,X)^{-1} f(X)$. This method is quite general since an appropriate choice of the kernel allows to recover the models obtained from various frameworks such as linear regression and splines models~\cite{wahba1990spline}.

\paragraph{}
One limitation of such models is the computational time required for building models based on a large number of observations. Indeed, these models require computing and inverting the $n \times n$ covariance matrix $k(X,X)$ between the observed values $\Y(X)$, which leads to a $O(n^2)$ complexity in space and $O(n^3)$ in time. In practice, this computational burden makes Gaussian process regression difficult to use when the number of observation points is in the range $[10^3,10^4]$ or greater.

\paragraph{}
Many methods have been proposed in the literature to overcome this limit. Let us first mention that, when the observations are recorded on a regular grid, choosing a separable covariance function $k$ enables to drastically simplify the inversion of the covariance matrix $k(X,X)$, since the latter can be written as a Kronecker product. In the same context of gridded data, alternative approaches such as Gaussian Markov Random Fields are also available~\cite{rue05gaussian}.

\paragraph{}
For irregularly spaced data, a common approach in machine learning relies on inducing points. It consists in introducing a set $W$ of pseudo input points and in approximating the full conditional distribution $\Y(x)|\Y(X)$ by $\Y(x)|\Y(W)$. The challenge here is to find the best locations for the inducing inputs and to decide which values should be assigned to the outputs at $W$. Various methods are suggested in the literature to answer these questions \cite{guhaniyogi2011adaptive,hensman2013,katzfuss2013bayesian,zhang2015full}. One drawback of this kind of approximation is that the predictions do not interpolate the observation points any more. Note that this method has recently been combined with the Kronecker product method in~\cite{nickson2015blitzkriging}.

\paragraph{}

Other methods rely \did{on} low rank approximations or compactly supported covariance functions. Both \did{methods} show limitations when the scale dependence is respectively short and large. For more details and references, see \cite{stein14limitations,maurya16well,bachoc2017}. Another
drawback – which to the best of our knowledge is little discussed in the literature – is the difficulty to use these methods when the dimension of the input space is large (say larger than 10, which is frequent in computer experiments or machine learning).

\paragraph{}
Let us also mention that the computation of $k(X,X)^{-1} y$, for an arbitrary vector $y\in \did{\R}^n$ can be performed using iterative algorithms, like the preconditionned conjugate gradient algorithm\linebreak \cite{golub12matrix}. Unfortunately, the algorithms need to be run many times when a posterior variance -- involving the computation of $k(X,X)^{-1} k(X,x_i)$ -- needs to be computed for a large set of prediction points.

\paragraph{}
The method proposed in this paper belongs to the so-called ``mixture of experts'' family. The latter relies on the aggregation of sub-models based on subsets of the data which make them easy to compute. This kind of methods offers a great flexibility since it can be applied with any covariance function and in large dimension while retaining the interpolation property. Some existing ``mixture of experts'' methods are product of experts \cite{hinton2002training}, and the (robust) Bayesian committee machine \cite{trespBCM,deisenroth2015}. All these methods are based on a similar approach: for a given point $x$, each sub-model provides its own prediction (a mean and a variance) and these predictions are then merged into one single mean and prediction variance. The  differences between these methods lie in how to aggregate the predictions made by each sub-model.
It shall be noted that aggregating expert opinions is the topic of \textit{consensus statistical methods} (sometimes referred to as \textit{opinion synthesis} or \textit{averaging methods}), where probability distributions representing expert opinions are joined together. Early references are \cite{winkler1968consensus,winkler1981combining}. A detailed review and an annotated bibliography is given in \cite{genest1986combining} (see also~\cite{satopaa2015modeling,ranjan2010combining} for recent related developments). From a probabilistic perspective, usual mixture of experts methods assume that there is some (conditional) independence between the sub-models. Although this kind of hypothesis leads to efficient computations, it is often violated in practice and may lead to poor predictions as illustrated in \cite{samo2016string}. Furthermore, these methods only provide pointwise confidence intervals instead of a full Gaussian process posterior distribution.

\paragraph{}
Since our method is part of the mixture of experts framework, it
benefits from the properties of the mixture of experts techniques: it
does not require the data to be on a grid, the predictions can interpolate the observations and it can be applied to data with small or large scale dependencies regardless of the input space dimension. Compared to other mixtures of experts, we relax the usually made independence assumption so that the prediction takes into account all $n^2$ pairwise cross-covariances between observations. We show that this addresses two main pitfalls of usual mixture of experts. First, the predictions are more accurate. Second, the theoretical consistency is ensured whereas it is not the case for the product of experts and the Bayesian committee machine methods. The detailed proofs of the later are out of the scope of this paper and we refer the interested reader to~\cite{bachoc2017} for further details.
The proposed method remains computationally affordable:  predictions are performed in a few seconds for $n = 10^4$ and a few minutes for $n=10^5$ using a standard laptop and the proposed online implementation. Finally, the prediction method comes with a naturally associated inference procedure, which is based on cross validation errors.

\paragraph{}
The proposed method is presented in Section~\ref{sec:2aggregation}. 
In Section~\ref{sec:3iterativescheme}, we introduce an iterative scheme for nesting the predictors derived previously. A procedure for estimating the parameters of models is then given in Section~\ref{sec:6estimation}. Finally, Section~\ref{sec:5discussion} compares the method with state-of-the-art aggregation methods on both a simulated dataset and an industrial case study.


\section{\did{Pointwise aggregation of experts}}
\label{sec:2aggregation}

\label{subsec:2a_pointwise}

\paragraph{}
Let us now address in more details the framework of this article. The method is based on the aggregation of sub-models defined on smaller subsets of points. Let $X_1,\ \dots,\ X_p$ be subvectors of the vector of observations input points $X$, it is thus possible to define $p$ associated sub-models (or \emph{experts}) $\M_1,\ \dots,\ \M_p$. For example, the sub-model $\M_i$ can be a Gaussian process regression model based on a subset of the data
\begin{equation}
	\M_i(x) = \Esp{\Y(x)|\Y(X_i)} = k(x,X_i) k(X_i,X_i)^{-1} \Y(X_i) \, ,\\
\label{eq:sub-models}
\end{equation}
however, we make no Gaussian assumption in this section. For a given prediction point $x\in D$, the $p$ sub-models predictions are gathered into a $p \times 1$ vector $M(x)=\transpose{(\M_1(x) \dots, \M_p(x))}$. The random column vector $\transpose{(\M_1(x),\ldots, \M_p(x),\Y(x))}$ is supposed to be centered with finite first two moments and we consider that both the $p \times 1$ covariance vector $\kxM=\Cov{M(x), Y(x)}$ and the $p \times p$ covariance matrix $\KM=\Cov{M(x),M(x)}$ are given.
%
Sub-models aggregation (or mixture of experts) aims at merging all the pointwise sub-models $M_1(x), \ldots, M_p(x)$ into one unique pointwise predictor $M_\agg(x)$ of $\Y(x)$. We propose the following aggregation:
\begin{samepage}
\begin{definition}[Sub-models aggregation]\label{def:agg}
For a given point $x \in D$, let $M_i(x)$, $i \in \agg =\set{1, \ldots, p}$ be sub-models with covariance matrix $\KM$. Then, when $\KM$ is invertible, we define the sub-model aggregation as:
\begin{equation}
	\M_\agg(x) =  \transpose{\kxM} \KM^{-1} \M(x).
	\label{eq:agg}
\end{equation}
\end{definition}
\end{samepage}
In practice, the invertibility condition on $\KM$ can be avoided by using matrices pseudo-inverses. Given the vector of observations $\M(x)=m(x)$, the associated prediction is
\begin{equation}
	m_\agg(x) = \transpose{\kxM} \KM^{-1} m(x).
\end{equation}
Notice that we are here aggregating random variables rather than their distributions. For dependent non-elliptical random variables, expressing the probability density function of $\M_\agg(x)$ as a function of each expert density $M_i(x)$ is not straightforward. This difference in the approaches implies that the proposed method differs from usual consensus aggregations. For example, aggregating random variables allows to specify the correlations between the aggregated prediction and the experts whereas aggregating expert distributions into a univariate prediction distribution does not characterize uniquely these correlations.
\begin{proposition}[BLUP]
	$\M_\agg(x)$ is the best linear unbiased predictor of $\Y(x)$ that writes $\sum_{i \in \agg} \alpha_i(x) \M_i(x)$. The mean squared error $v_\agg(x)=\Esp{(Y(x)-\M_\agg(x))^2}$ writes
	\begin{equation}
		v_\agg(x) = k(x,x) - \transpose{\kxM} \KM^{-1} \kxM \, .
	\end{equation}
The coefficients $\set{\alpha_i(x), i\in \agg}$ are given by the vector $\alpha=\transpose{\kxM} \KM^{-1}$.
\end{proposition}
\begin{proof}
The standard proof applies: The square error writes $\Esp{(Y(x)-\transpose{\alpha}\M(x))^2} = k(x,x) - 2 \transpose{\alpha} \kxM + \transpose{\alpha} \KM \alpha$. The value of $\alpha^*$ minimising it can be found by differentiation: $-2 \kxM + 2 \alpha^* \KM = 0$ which leads to $\alpha^* =  \KM^{-1} \kxM$. Then, $v_\agg(x) = k(x,x) - 2 \alpha^{*t} \kxM + \alpha^{*t} \KM \alpha^*$ and the result holds.
\end{proof}
\begin{proposition}[Basic properties]\label{prop:basicProperty}
Let $x$ be a given prediction point in $D$.
\begin{itemize}
	\item[(i)] \label{itm:prop:Lineaire} Linear case: if $\M(x)$ is linear in $Y(X)$, \ie if there exists a $p \times n$ deterministic matrix $\Lambda(x)$ such that $\M(x) = \Lambda(x) \Y(X)$ and if  $\Lambda(x) k(X,X) \transpose{\Lambda(x)}$ is invertible, then $M_\agg(x)$ is linear in $Y(X)$ with
	\begin{equation}
	\accolade{
		\M_\agg(x) &=& \transpose{\lambda_\agg(x)} \Y(X) \virguleacc\\
v_\agg(x) &=&	k(x,x) - \transpose{\lambda_\agg(x)} k(X,x)	\pointacc
		}
	\end{equation}
where $\transpose{\lambda_\agg(x)} = k(x,X) \transpose{\Lambda(x)}  \big(\Lambda(x) k(X,X) \transpose{\Lambda(x)} \big)^{-1} \Lambda(x)$.
\item[(ii)] Interpolation case: if $\M$ interpolates $\Y$ at $X$, \ie if for any component $x_k$ of the vector $X$ there is at least one index $i_k \in \agg$ such that $M_{i_k}(x_k)=Y(x_k)$, and if $K_M(x_k)$ is invertible for any component $x_k$ of $X$, then $\M_\agg$ is also interpolating, \ie
\begin{equation}
\accolade{
\M_\agg(X) &=& \Y(X) \virguleacc\\
v_\agg(X) &=& 0_n \virguleacc
}
\label{eq:interp}
\end{equation}
where $0_n$ is a $n \times 1$ vector with entries $0$. This property can be extended when some $K_M(x_k)$ are not invertible by using pseudo-inverse in place of matrix inverse in Definition~\ref{def:agg}.
\item[(iii)] Gaussian case: if the joint distribution $(\M(x),Y(x))$ is multivariate normal, then the conditional distribution of $\Y(x)$ given $\M(x)$ is normal with moments
	\begin{equation}
\accoladesplit{
			\Esp{\Y(x)|\Mpx} & = \M_\agg(x) \virguleacc  \\
			\Var{\Y(x)|\Mpx} & = v_\agg(x) \pointacc
}
		\label{eq:modelagg}
	\end{equation}
\end{itemize}
\end{proposition}
\begin{proof}
Linearity directly derives from  $k_M(x)=\Lambda(x) k(X,x)$ and $K_M(x)=\Lambda(x) K(X,X) \transpose{\Lambda(x)}$. \\
Interpolation: Let $k \in \set{1, \ldots, n}$, and  $i\in \agg$ be an index such that $M_i(x_k)=Y(x_k)$. As $K_M(x_k)=\Cov{M(x_k),M(x_k)}$, the $i^{th}$ line of $K_M(x_k)$ is equal to $\Cov{M_i(x_k), M(x_k)}=\Cov{Y(x_k),M(x_k)}=\transpose{k_M(x_k)}$. Setting $e_i$ the $p$ dimensional vector having entries $0$ except on its $i^{th}$ component, it is thus clear that $\transpose{e_i} K_M(x_k) = \transpose{k_M(x_k)}$. As $K_M(x_k)$ is assumed to be invertible, then $\transpose{e_i}  = \transpose{k_M(x_k)} K_M(x_k)^{-1}$, so that $M_\agg(x_k)=\transpose{k_M(x_k)} K_M(x_k)^{-1} M(x_k) = \transpose{e_i} M(x_k)= M_i(x_k)=Y(x_k)$. This result can be plugged into the definition of $v_\agg$ to obtain the second part of Eq.~\eqref{eq:interp}: $v_\agg(x_k) = \Esp{(Y(x_k)-\M_\agg(x_k))^2}=0$.\\
Finally the Gaussian case can be proved directly by applying the usual multivariate normal conditioning formula.
\end{proof}

\paragraph{}
\did{
Let us assume here that conditions in items~(i) and~(ii) of Proposition~\ref{prop:basicProperty} are satisfied, that is that $\M(x)$ is linear in $\Y(X)$, and that $\M_\agg(X)=\Y(X)$. Then, the proposed aggregation \did{method} also benefits from several other interesting properties:
\begin{itemize}
\item[$\bullet$] First, the aggregation can be seen as an exact conditional process for a slightly different prior distribution on $Y$. One can indeed define a process $\Y_\agg$ as $\Y_\agg = \M_\agg + \varepsilon'_\agg $ where $\varepsilon'_\agg$ is an independent replicate of $\Y - \M_\agg$ and with $M_{\agg}$ as in \eqref{eq:agg}. One can then show that $\Y_\agg(X)=\Y(X)$ and
\begin{equation*}
\accoladesplit{
		\M_\agg(x) & = \Esp{\Y_\agg(x) | \Y_\agg(X)} \virguleacc\\
		v_\agg(x) & = \Var{\Y_\agg(x) | \Y_\agg(X)} \pointacc
}
\end{equation*}
Denoting $k_\agg$ the covariance function of $\Y_\agg$, one can also show that $k_\agg(x,x) = k(x,x)$ for all $x \in D$ and $k_\agg(X,X) = k(X,X)$.
\item[$\bullet$] Second, the error between the aggregated model and the full model of Equation~\eqref{eq:fullmodel} can be bounded. For any norm $\norm{.}$, one can show that there exists some constants $\lambda, \mu \in \R^+$ such that
	\begin{equation*}
		\accolade{
			\abs{\M_\agg(x)-M_\full(x)} & \le & \lambda \norm{k(X,x)} \norm{Y(X)} \virguleacc\\
			\abs{v_\agg(x)-v_\full(x)} & \le & \mu \norm{k(X,x)}^2 \pointacc
		}
	\end{equation*}
One can also show that the differences between the full and aggregated models write as norm differences, where $\norm{u}_K^2 = \transpose{u} k(X,X)^{-1} u$:
\begin{equation*}
\accoladesplit{
	\Esp{(\M_\agg(x) - \M_\full(x))^2} &= \norm{k(X,x)-k_\agg(X,x)}_K^2 \virguleacc\\
	v_\agg(x) - v_\full(x) &= \norm{k(X,x)}_K^2 - \norm{k_\agg(X,x)}_K^2 \pointacc
}
\label{eq:diff}
\end{equation*}
\item[$\bullet$] Third, contrarily to several other aggregation methods, when sub-models are informative enough, the difference between the aggregated model and the full model vanishes: when $\M(x) = \Lambda(x) \Y(X)$ where $\Lambda(x)$ is a $n\times n$ matrix with full rank, then $\Y_\agg \stackrel{law}{=} \Y$ and $\Y_\agg |  \Y_\agg(X) \stackrel{law}{=} \Y | \Y(X)$. Furthermore, in this full-information case,
	\begin{equation*}
	\accolade{
		\M_\agg(x) &=& M_\full(x) \virguleacc\\
		v_\agg(x) &=& v_\full(x) \pointacc
		}
		\label{eq:fullyinformative}
	\end{equation*}
\item[$\bullet$] Finally, in the Gaussian case and under some supplementary conditions, it can be proven that, contrarily to several other aggregation methods, the proposed method is consistent when the number of observation points tends toward infinity. Let $(x_{ni})_{1 \leq i \leq n, n \in \mathds{N}}$ be a triangular array of observation points such that for all $x \in D$, $\lim_{n \to \infty} \min_{i=1,...,n} || x_{ni} - x || = 0$. For $n \in \mathds{N}$, let $X=(x_{n1},...,x_{nn})^t$, let $\M_{1}(x),...,\M_{p_n}(x)$ be any collection of $p_n$ simple Kriging predictors based on respective design points $X_1,\ \dots, \ X_{p_n}$ where $\cup_{i=1}^{p_n} X_i=X$ (with a slight abuse of notation), then
\begin{equation*}
\sup_{x \in D} \mathbb{E} \left(
\left(
Y(x) - M_{\agg}(x)
\right)^2
\right)
\to_{n \to \infty} 0.
\end{equation*}
One can also exhibit non-consistency results for other aggregation methods of the literature that do not use covariances between sub-models.
\end{itemize}
}
The full development of these properties is out of the scope of this paper and we dedicate a separate article to detail them~\cite{bachoc2017}\did{.\\}

We now illustrate our aggregation method with two simple examples.

\begin{example}[Gaussian process regression aggregation]\label{example:Kriging}
In this example, we set $D=\R$ and we approximate the function $f(x)=\sin(2 \pi x) +x$ based on  a set of five observation points in $D$: $\set{0.1, 0.3, 0.5, 0.7, 0.9}$. These observations are gathered in two column vectors $X_1 = (0.1, 0.3, 0.5)$ and $X_2 = (0.7, 0.9)$.
We use as prior a centered Gaussian process $\Y$ with squared exponential covariance $k(x,x') = \exp \left( -12.5 (x-x')^2 \right)$ in order to build two Kriging sub-models, for $i \in \set{1, 2}$:
\begin{equation}\label{suppr1}
\accoladesplit{
		\M_i(x) &= \Esp{\Y(x)|\Y(X_i)} = k(x,X_i)k(X_i,X_i)^{-1} \Y(X_i) \virguleacc \\
		m_i(x) &= \Esp{\Y(x)|\Y(X_i) \shorteq f(X_i)} = k(x,X_i)k(X_i,X_i)^{-1}f(X_i) \pointacc
}
\end{equation}
The expressions required to compute $\M_\agg$ as defined in Eq.~\eqref{eq:agg} are for $i,j \in \set{1, 2}$:
\begin{equation}\label{suppr2}
\accoladesplit{
		\big(\kxM \big)_i &= \Cov{M_i(x),Y(x)} = k(x,X_i)k(X_i,X_i)^{-1} k(X_i,x)  \virguleacc \\
		\big(\KM \big)_{i,j} &= \Cov{M_i(x),M_j(x)} = k(x,X_i)k(X_i,X_i)^{-1} k(X_i,X_j) k(X_j,X_j)^{-1} k(X_j,x) \pointacc
}
\end{equation}
Recall $m_\full(x) = \Esp{\Y(x)|\Y(X) \shorteq f(X)}$ and $v_\full(x) = \Var{\Y(x)|\Y(X) \shorteq f(X)}$, as it can be seen in Figure~\ref{fig:ex2}, the resulting model $m_\agg$ appears to be a very good approximation of $m_\full$ and there is only a slight difference between prediction variances $v_\agg$ and $v_\full$ on this example.
\begin{figure}
\centering
  \subfloat[sub-models to aggregate]{\includegraphics[height=5cm]{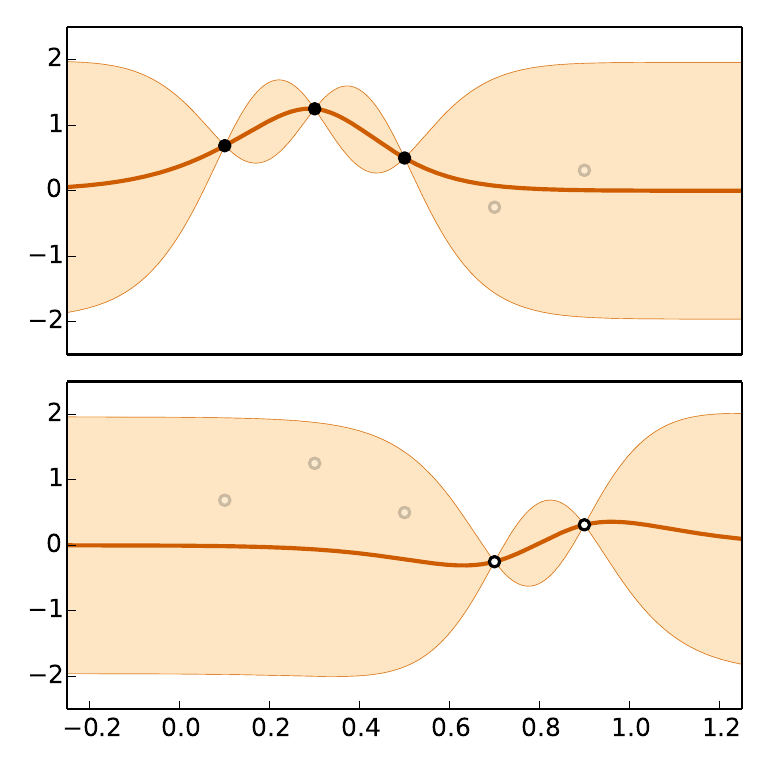}}
  \subfloat[aggregated model (solid lines) and full model (dashed lines)]{\includegraphics[height=5cm]{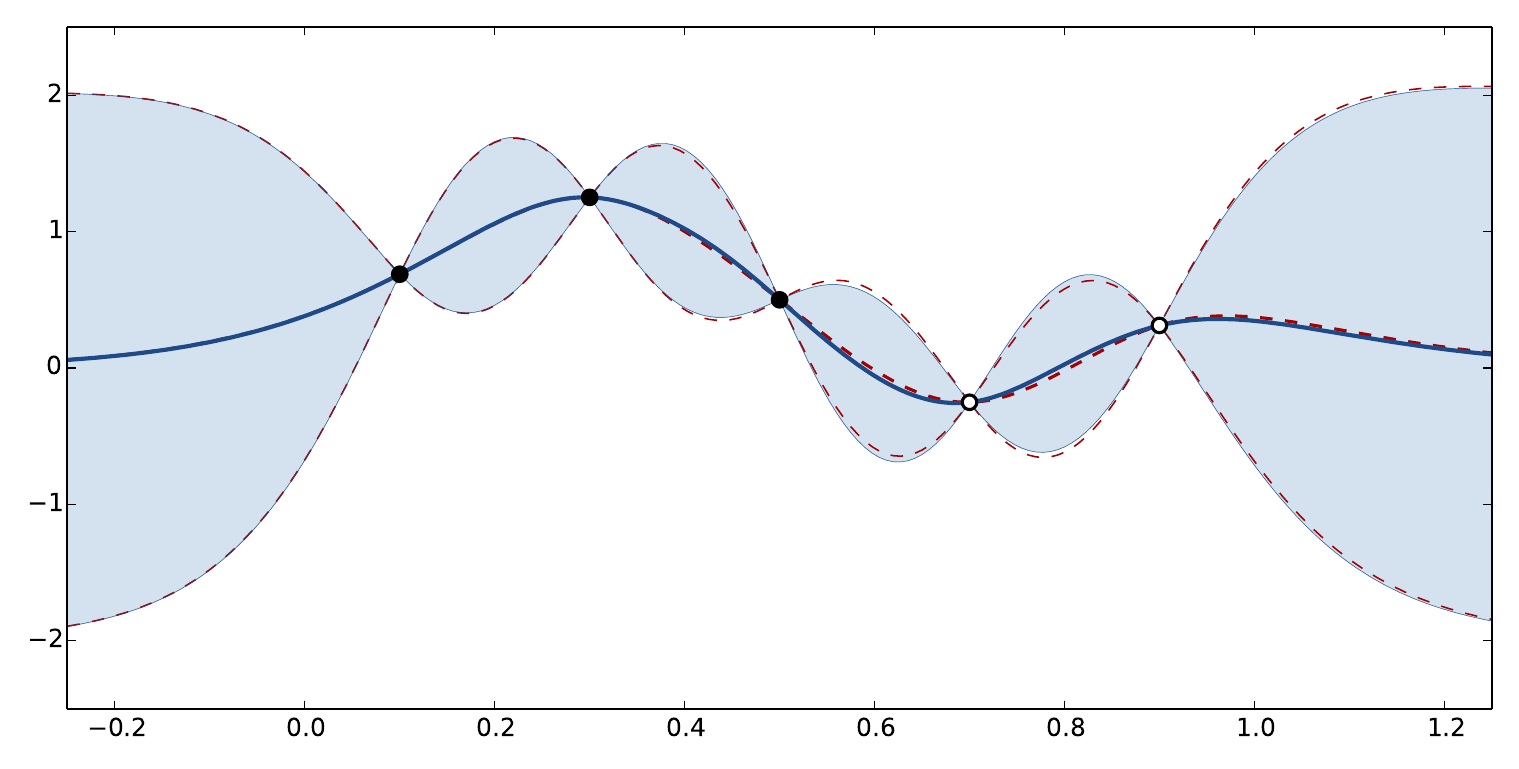}}
  \caption{Example of aggregation of two Gaussian process regression models. For each model, we represent the predicted mean and $95\%$ confidence intervals.}
  \label{fig:ex2}
\end{figure}
\end{example}
\begin{example}[Linear regression aggregation]
In this distribution-free example, we set $D=\R$ and we consider the process $\Y(x) = \varepsilon_1 + \varepsilon_2 x $ where $\varepsilon_1$ and $\varepsilon_2$ are independent centered random variables  with unit variance. $\Y$ is thus centered with covariance $k(x,x')=1+xx'$. Furthermore, we consider that $\Y$ is corrupted by some observation noise $\Y_{obs}(x) = \Y(x) + \varepsilon_3(x)$ where $\varepsilon_3(x)$ is an independent white noise process with covariance $k_3(x,x')=\Indic{x=x'}$. Note that we only make assumptions on the first two moments of $\varepsilon_1$, $\varepsilon_2$ or $\varepsilon_3(x)$ but not on their laws.
We introduce five observation points gathered in two column vectors: $X_1 =\transpose{(0.1, 0.3, 0.5 )}$ and $X_2 = \transpose{(0.7, 0.9 )}$ and their associated outputs $y_1 = \transpose{(2.05, 0.93, 0.31)}$ and $y_2 = \transpose{(-0.47,0.12 )}$.
The linear regression sub-models, obtained by square error minimization, are $M_i(x)=k(x,X_i)(k(X_i,X_i)+ \Id)^{-1}\Y_{obs}(X_i)$, $i \in \set{1,2}$, where $\Id$ stands for the identity matrix. Resulting covariances $\Cov{\M_i(x),\Y(x)}$, $\Cov{\M_i(x),\M_j(x)}$ and aggregated model $\M_\agg(x)$, $v_\agg(x)$ of Eq.~\eqref{eq:modelagg} are then easily obtained.
%
The resulting model is illustrated in Figure~\ref{fig:ex1}.
\begin{figure}
\centering
  \subfloat[sub-models to aggregate]{\includegraphics[height=5cm]{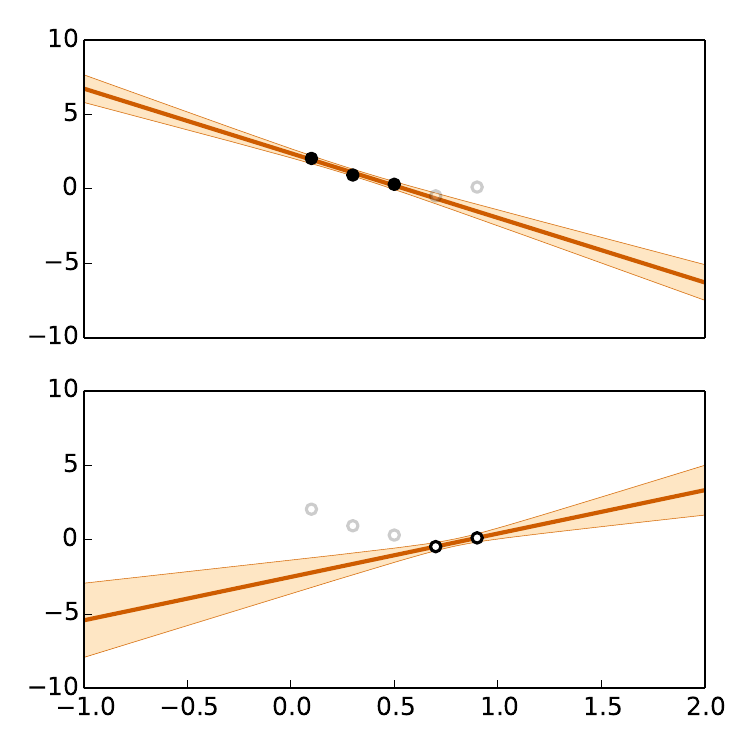}}
  \subfloat[aggregated model (solid lines) and full model (dashed lines)]{\includegraphics[height=5cm]{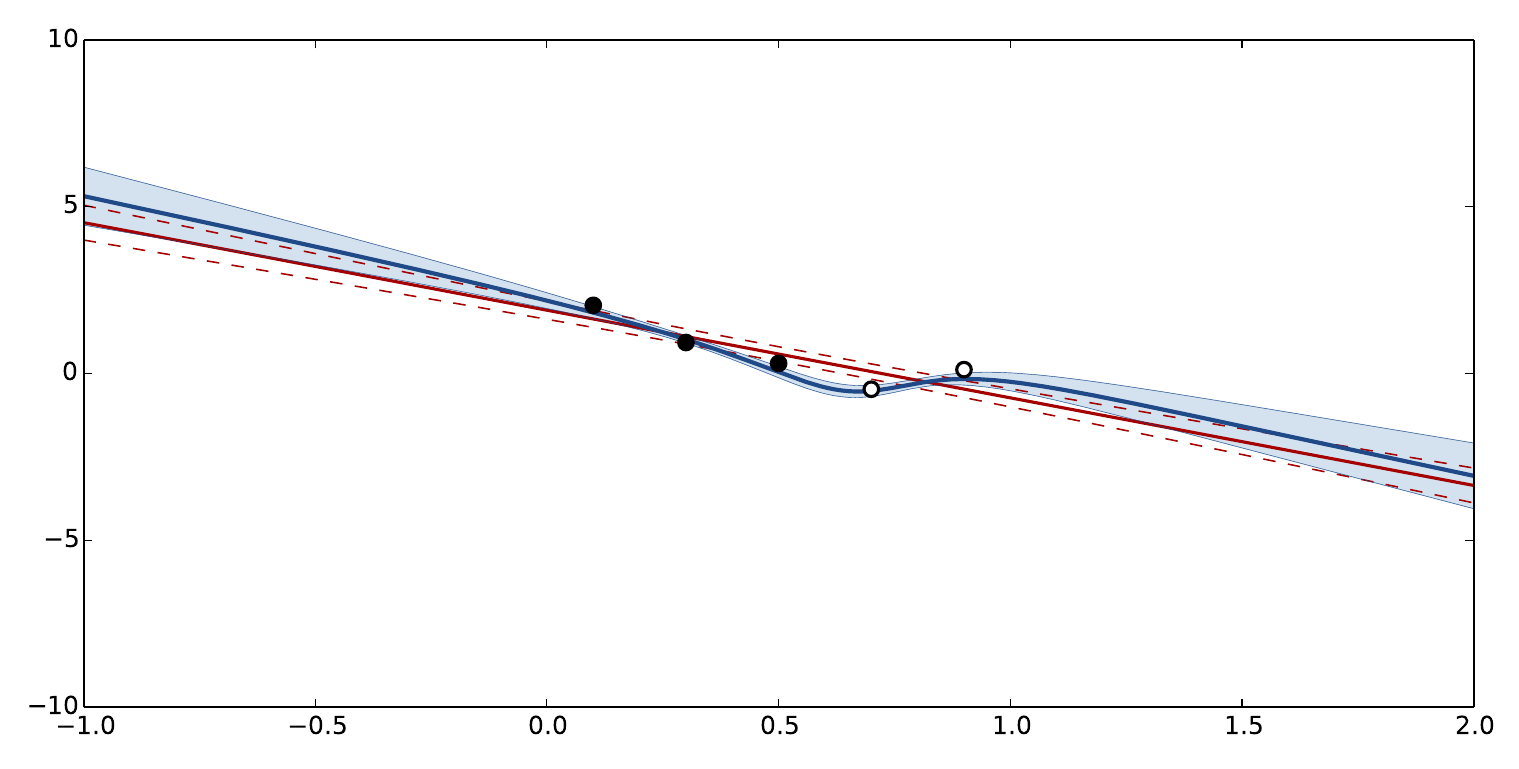}}
	\caption{Example of aggregation of two linear regression sub-models. Exhibited confidence bands correspond to a difference to mean value of two standard deviations.}
	\label{fig:ex1}
\end{figure}
\end{example}


\section{Iterative scheme}
\label{sec:3iterativescheme}

In the previous sections, we have seen how to aggregate sub-models $\M_1, \ldots, \M_p$ into one unique aggregated value $M_\agg$. Now, starting from the same sub-models, one can imagine creating several aggregated values, $M_{\agg_1}, \ldots, M_{\agg_\ns}$, each of them based on a subset of $\set{\M_1, \ldots, \M_p}$. One can show that these aggregated values  can themselves be aggregated. This makes possible the construction of an iterative algorithm that merges sub-models at successive steps, according to a tree structure. Such tree-based schemes are sometimes used to reduce the complexity of models, see \eg~\cite{tzeng2005fast}, or to allow parallel computing~\cite{doi:10.1080/15481603.2014.1002379}.\\

The aim of this section is to give a generic algorithm for aggregating sub-models according to a tree structure and to show that the choice of the tree structure helps partially reducing the complexity of the algorithm. It also aims at giving perspectives for further large reduction of the global complexity.

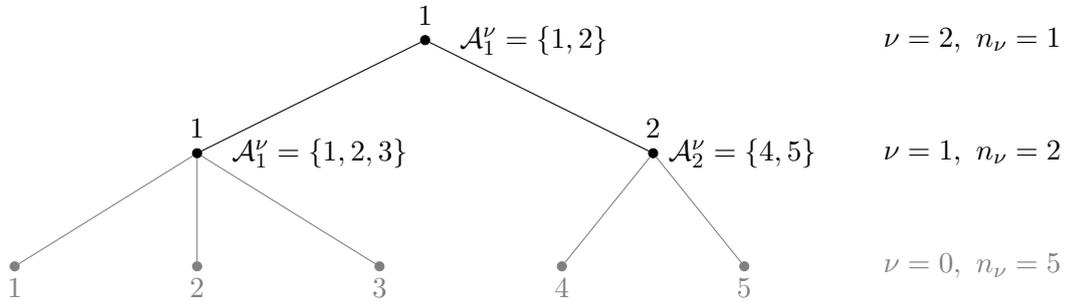
\begin{figure}
\begin{center}
\begin{tikzpicture}[]
  \tikzset{
    level 1/.style={level distance=15mm,sibling distance=60mm},
    level 2/.style={level distance=15mm,sibling distance=24mm},
    level 3/.style={level distance=15mm,sibling distance=12mm},
    solid node/.style={circle,draw,inner sep=1.2,fill=black},
    leave node/.style={circle,draw,inner sep=1.2,fill=gray,color=gray,edge from parent/.style={gray,draw}}
  }

  \node[solid node,label=above:{1}, label=right:{$\ \ \setI_1^\nu = \{ 1,2 \}$}]{}
    child{node[solid node,label=above:{1},label=right:{$\ \ \setI_1^\nu = \{ 1,2,3 \}$} ]{}
      child[leave node]{node[leave node,label=below:{$1$}]{}}
      child[leave node]{node[leave node,label=below:{$2$}]{}}
      child[leave node]{node[leave node,label=below:{$3$}]{}}
    }
    child{node[solid node,label = above:{2}, label=right:{$\setI_2^\nu = \{ 4,5 \}$}]{}
      child[leave node]{node[leave node,label=below:{$4$}]{} edge from parent[gray]}
      child[leave node]{node[leave node,label=below:{$5$}]{}
        child [grow=right] {node (inv) {} edge from parent[draw=none]
          child [grow=right] {node (a) {$\nu=0,\ n_\nu=5$} edge from parent[draw=none]
            child [text=black,grow=up] {node (b) {$\nu=1,\ n_\nu=2$} edge from parent[draw=none]
              child [text=black,grow=up] {node (c) {$\nu=2,\ n_\nu=1$} edge from parent[draw=none]
              }
            }
          }
        }
      }
    }
  ;
\end{tikzpicture}
\end{center}
\caption{One aggregation tree with height $\numax=2$, $n_0=5$ initial leave nodes (observation points) and $n_1=2$ sub-models.}
\label{fig:tree}
\end{figure}

\paragraph{}
Let us introduce some notations. The total height (i.e number of layers) of the tree is denoted $\numax$ and the number of node of a layer $\nu \in \set{1,\dots,\numax}$ is $n_\nu$. We associate to each node (say node~$i$ in layer~$\nu$) a sub-model $M_i^\nu$ corresponding to the aggregation of its child node sub-models. In other words, $M_i^\nu$ is the aggregation of $\set{M_k^{\nu-1}, \,k \in \agg_i^\nu}$ where $\agg_i^\nu$ is the set of children of node~$i$ in layer~$\nu$. These notations are summarized in Figure~\ref{fig:tree} which details the tree associated with Example~\ref{example:Kriging}. In practice, there will be one \emph{root node} ($n_{\numax}=1$) and each node will have at least one parent: $\cup_{i = 1,\dots,n_\nu} \setI_i^\nu = \set{ 1,\dots,n_{\nu-1}}$. Typically, the sets $\setI_i^\nu$, $i ={1,\ldots,n_\nu}$, are a partition of $\set{ 1,\dots,n_{\nu-1}}$ but this assumption is not required and a child node may have several parents (which can generate a lattice rather than a tree).

\subsection{Two-Layer aggregation}

\paragraph{}
We discuss in this section the tree structure associated \did{with} the case $\numax=2$ as per the previous examples. With such settings, the first step consists in calculating the initial sub-models $M^1_1, \ldots, M^1_p$ of the layer $\nu=1$ and the second one is to aggregate all sub-models of layer $\nu=1$ into one unique predictor $M^2_1$ (see for example Figure~\ref{fig:tree}). This aggregation is obtained by direct application of Definition~\ref{def:agg}.

\FloatBarrier

In practice the sub-models can be any covariates, like gradients, non-Gaussian underlying factors or even  black-box responses, as soon as cross-covariances and covariances with $Y(x)$ are known. When sub-models are calculated from direct observations $Y(x_1), \ldots, Y(x_n)$, the number of leave nodes at layer $\nu=0$ is $n_0=n$.
In further numerical illustrations of Section~\ref{sec:5discussion}, the sub-models $\M_i^1$ are simple Kriging predictors of $\Y(x)$, with for $i=1,\ldots, p$, 
\begin{equation}\label{eq:Krigingsub-modelsIterative}
\accolade{
M_i^{1}(x) &=& k(x,X_i) k(X_i, X_i)^{-1} Y(X_i) \virguleacc\\
\Cov{M_i^{\did{1}}(x), Y(x)} &=& k(x,X_i) k(X_i, X_i)^{-1} k(X_i,x) \virguleacc\\
\Cov{M_i^{\did{1}}(x), M_j^{\did{1}}(x)} &=& k(x,X_i) k(X_i, X_i)^{-1} k(X_i, X_j) k(X_j, X_j)^{-1} k(X_j,x)\pointacc
}
\end{equation}
With these particular simple Kriging initial sub-models, the layer $\nu=1$ corresponds to the aggregation of covariates $\M_i^0(x)=\Y(x_i)$ at the previous layer $\nu=0$, $i=1,\ldots, n$.

\subsection{Multiple Layer aggregation}

In order to extend the two-layer settings, one needs to compute covariances among aggregated sub-models. The following proposition gives covariances between aggregated models of a given layer.

\begin{proposition}[aggregated models covariances]\label{prop:covariancesM}
Let us consider a layer $\nu \ge 1$ and given aggregated models $M_1^\nu(x), \ldots, M_{n_{\nu}}^\nu(x)$. Assume that the following covariances $(k^\nu(x))_i =\Cov{M_i^\nu(x), Y(x)}$ and $(K^\nu(x))_{ij}=\Cov{M_i^\nu(x), M_j^\nu(x)}$ are given, $i,j \in \set{1, \ldots, n_\nu}$.
Let $n_{\nu+1} \ge 1$ be a number of new aggregated values. Consider subsets $\agg_i^{\nu+1}$ of $\set{1, \ldots, n_\nu}$, $i =1, \ldots, n_{\nu+1}$, and assume that $M_{i}^{\nu+1}(x)$ is the aggregation of $M_k^\nu(x)$, $k \in \agg_i^{\nu+1}$. Then
\begin{equation}\label{eq:tripletEtapeSuivante}
\accolade{
(M^{\nu+1}(x))_i &=& \transpose{\alpha_i^{\nu+1}(x)} \left( M^{\nu}(x)_{[\setI_i^{\nu+1}]} \right) \virguleacc\\
\Cov{M_i^{\nu+1}(x), Y(x)} &=&  \transpose{\alpha_i^{\nu+1}(x)} \left( k^{\nu}(x)_{[\setI_i^{\nu+1}]} \right) \virguleacc\\
\Cov{M_i^{\nu+1}(x), M_j^{\nu+1}(x)} &=& \transpose{\alpha_i^{\nu+1}(x)} \left( K^{\nu}(x)_{[\setI_i^{\nu+1},\setI_j^{\nu+1}]} \right)  \alpha_j^{\nu+1}(x) \virguleacc
}
\end{equation}
where the vectors of optimal weights are
$\alpha_i^{\nu+1}(x) = \left( K^{\nu}_{[\setI_i^{\nu+1},\setI_i^{\nu+1}]} \right)^{-1} \left( k^{\nu}(x)_{[\setI_i^{\nu+1}]} \right)$ and where $k^\nu(x)_{[\setI_i^{\nu+1}]}$ corresponds to the sub-vector of $k^\nu(x)$ of indices in $\setI_i^{\nu+1}$ and similarly for $\M^\nu(x)_{[\setI_i^{\nu+1}]}$ and the submatrix $K^\nu(x)_{[\setI_i^{\nu+1},\setI_i^{\nu+1}]}$, which is assumed to be invertible.\\
Furthermore, $\Cov{M_i^{\nu+1}(x), Y(x)}=\Cov{M_i^{\nu+1}(x), M_i^{\nu+1}(x)}$.
\end{proposition}
\begin{proof}
This follows immediately from Definition~\ref{def:agg}: as the aggregated values are linear expressions, the calculation of their covariances is straightforward. The last equality is simply obtained by inserting the value of $\alpha_i^{\nu+1}(x)$ into the expression of $\Cov{M_i^{\nu+1}(x), M_i^{\nu+1}(x)}$.
\end{proof}

The following algorithm, which is a generic algorithm for aggregating sub-models according to a tree structure, is based on an iterative use of the previous proposition. It is given for one prediction point $x\in D$ and it assumes that the sub-models are already calculated, starting directly from layer 1. This allows a large variety of sub-models, and avoids the storage of the possibly large covariance matrix $K^{0}(x)$.
Its outputs are the final scalar aggregated model, $M_{\numax}(x)$, and the scalar covariance $K_{\numax}(x)$ from which one deduces the prediction error $\Esp{(Y(x)-M_{\numax}(x))^2}=k(x,x)-K_{\numax}(x)$.\\

In order to give dimensions in the algorithm and to ease the calculation of complexities, we define $c_i^\nu$ as the number of children of the sub-model $M_i^\nu$, $c_i^\nu=\mathrm{card} \, \setI_i^\nu$. We also denote $c_{\max} = \mathop{\max}\limits_{\nu, i}c_i^\nu$ the maximal number of children.\\

\begin{algorithm}[H]
\DontPrintSemicolon
\CommentSty{\color{blue}}

\SetKwInOut{Input}{inputs}
\SetKwInOut{Output}{outputs}
\Input{$M_1$, vector of length $ n_1$ (sub-models evaluated at $x$) \newline
$k_1$, vector of length $ n_1$ (covariance between $\Y(x)$ and sub-models at $x$) \newline
$K_1$, matrix of size $ n_1 \times n_1$ (covariance between sub-models at $x$) \newline
$\setI$, a list describing the tree structure}

\Output{$M_{\numax}$, $K_{\numax}$}
\medskip

\textcolor{gray}{Create vectors $M$, $k$ of size $c_{\max}$ and matrix $K$ of size $c_{\max} \times c_{\max}$ \;}

\For{$\nu=2, \ldots, \numax$}
{
\textcolor{gray}{Create vectors $M_{\nu}$ of size $n_{\nu}$ and matrix $K_{\nu}$ of size $n_{\nu} \times n_{\nu}$ \;}

\For{$i=1, \ldots, n_{\nu}$}
{
\textcolor{gray}{Create vector $\alpha_i$ of size $c_i^{\nu}$\;}
$M \leftarrow$ subvector of $M_{\nu-1}$ on $\setI_i^\nu$\;
$K \leftarrow$ submatrix of $K_{\nu-1}$ on $\setI_i^\nu$\;
\KwSty{if } $\nu=2$ \KwSty{ then } $k \leftarrow k_1$ \KwSty{else } $k \leftarrow \mathrm{Diag}(K)$ \;
$\alpha_i \leftarrow K^{-1} k$\;
$M_{\nu}[i] \leftarrow  \transpose{(\alpha_i)} M$\;
$K_{\nu}[i,i] \leftarrow  \transpose{(\alpha_i)} k$ \;
\For{$j=1, \ldots, i-1$}
{
$K \leftarrow$ submatrix of $K_{\nu-1}$ on $\setI_i^\nu \times \setI_j^\nu$\;
$K_{\nu}[i,j] \leftarrow \transpose{(\alpha_i)} K \alpha_j$\;
$K_{\nu}[j,i] \leftarrow K_{\nu} [i,j]$\;
} 
} 
\textcolor{gray}{$M_{\nu-1}$, $K_{\nu-1}$ and all $\alpha_i$ can be deleted}
} 
\caption{Nested Kriging algorithm}
\label{algoKrigIteratif2}
\end{algorithm}

\paragraph{}
Notice that Algorithm~\ref{algoKrigIteratif2} uses the result $(K^{\nu+1}(x))_{ii} = (k^{\nu+1}(x))_i$ from Prop.~\ref{prop:covariancesM}: when we consider aggregated models ($\nu \ge 2$), we do not need to store and compute the vector $k^{\nu}(x)$ any more. When $\nu=1$, depending on the initial covariates, $\Cov{M^1_i(x), Y(x)}$ is not necessarily equal to  $\Cov{M^1_i(x), M^1_i(x)}$ (this is however the case when $M^1_i(x)$ are simple Kriging predictors).\\

For the sake of clarity, some improvements have been omitted in the algorithm above. For instance, covariances can be stored in triangular matrices, one can store two couples $(M_{\nu},K_{\nu})$ instead of $\numax$ couples by using objects $M_{(\nu\mod 2)}$ and $K_{(\nu\mod 2)}$. Furthermore, it is quite natural to adapt this algorithm to parallel computing, but this is out of the scope of this article.

\subsection{Complexity}

We study here the complexity of Algorithm~\ref{algoKrigIteratif2} in space (storage footprint) and in time (execution time). For the sake of clarity we consider in this paragraph a simplified tree where $n_\nu$ is decreasing in $\nu$ and each child has only one parent. This corresponds to the most common structure of trees, without overlapping. Furthermore, at any given level $\nu$, we consider that each node has the same number of children: $c_i^{\nu}=c_\nu$ for all $i=1, \ldots, n_\nu$. Such a tree will be called \textit{regular}. In this setting, one easily sees that $n_\nu= \frac{n_{\nu-1}}{c_\nu}= \frac{n}{c_1\ldots c_\nu} $, $\nu \in \set{1, \ldots, \numax}$. Complexities obviously depend on the choice of sub-models, we give here complexities for Kriging sub-models as in Eq.~\eqref{eq:Krigingsub-modelsIterative}, but this can be adapted to other \did{kinds} of sub-models.\\

For one prediction point $x \in D$, we denote by $\mathcal{S}$ the storage footprint of Algorithm~\ref{algoKrigIteratif2}, and by $\CalgoTot$ its complexity in time, including sub-models calculation. One can show that in a particular two-layers setting with $\sqrt{n}$ sub-models ($\numax=2$ and $c_1=c_2=\sqrt{n}$), a reachable global complexity for $q$ prediction points is (see assumptions below and expression details in the proof of Proposition~\ref{prop:complexities})
\begin{equation}
\mathcal{S}=O(n) \boxand q \CalgoTot  = O(n^2 q) \, .
\end{equation}
This is to be compared with $O(n^3)+O(n^2 q)$ for the same prediction with the full model. The  aggregation of sub-models can be useful when the number of prediction points is smaller than the number of observations. Notice that the storage needed for $q$ prediction points is the same as for one prediction point, but in some cases (as for leave-one-out errors calculation), it is worth using a $O(nq)$ storage to avoid recalculations of some quantities.\\

We now detail chosen assumptions on the calculation of $\mathcal{S}$ and $\CalgoTot$, and study the impact of the tree structure on these quantities.
For one prediction point $x \in D$, including sub-models calculation, the complexity in time can be decomposed into $\CalgoTot = \Ccov+\CalgoAlpha+\CalgoBeta$, where
\begin{itemize}
\item[-] $\Ccov$ is the complexity for computing all cross covariances among initial design points, which does not depend on the tree structure (neither on the number of prediction points).
\item[-] $\CalgoAlpha$ is the complexity for building all aggregation predictors, \ie the sum over $\nu, i$ of all operations in the $i$-loop in Algorithm~\ref{algoKrigIteratif2} (excluding operations in the $j$-loop).
\item[-] $\CalgoBeta$ is the complexity for building the covariance matrices among these predictors, \ie the  sum over $\nu, i, j$ of all operations in the $j$-loop in Algorithm~\ref{algoKrigIteratif2}.
\end{itemize}

We assume here that there exists two constants $\alpha>0$ and $\beta>0$ such that the complexity of operations inside the $i$-loop (excluding those of the $j$-loop) is $\alpha c_{\nu}^3$, and the complexity of operations inside the $j$-loop is $\beta c_{\nu}^2$. Despite perfectible, this assumption follows from the fact that one usually considers that the complexity of $c_\nu \times c_\nu$ matrix inversion is $O(c_\nu^3)$ and the complexity of matrix-vector multiplication is $O(c_\nu^2)$. We also assume that the tree height $\numax$ is finite, and that all numbers of children $c_\nu$ tend to $+\infty$ as  $n$ tends to $+\infty$. This excludes for example binary trees, but makes assumptions on complexities more reliable. Under these assumptions, the following proposition details how the tree structure affects the complexities.
\begin{proposition}[Complexities]\label{prop:complexities}
The following storage footprint $\mathcal{S}$ and complexities $\CalgoAlpha$, $\CalgoBeta$ hold for the respective tree structures, when the number of observations $n$ tends to $\infty$.
\begin{itemize}
\item[(i)] The two-layer equilibrated $\sqrt{n}$-tree, where $p=c_1=c_2=\sqrt{n}$, $\numax=2$, is the optimal storage footprint tree, and
\begin{equation}
\mathcal{S}=O(n) \,, \qquad \CalgoAlpha \sim \alpha n^{2}\,, \qquad \CalgoBeta \sim \frac{\beta}{2} n^2\,.
\end{equation}
\item[(ii)] The $\numax$-layer equilibrated $\sqrt[\numax]{n}$-tree, where $c_1=\dots=c_{\numax}=\sqrt[\numax]{n}$, $\numax \ge 2$, is such that
\begin{equation}
\mathcal{S}=O(n^{2-2/\numax}) \,, \qquad \CalgoAlpha \sim \alpha n^{1+\frac{2}{\numax}} \,, \qquad \CalgoBeta \sim \frac{\beta}{2} n^2\,.
\end{equation}
\item[(iii)] The optimal complexity tree is defined as the regular tree structure that minimizes $\CalgoAlpha$, as it is not possible to reduce  $\CalgoBeta$ to orders lower than $O(n^2)$. This tree is such that
\begin{equation}
\mathcal{S}=O\left(n^{2-\frac{1}{\delta^{\numax}-1}}\right)  \,, \qquad \CalgoAlpha \sim \gamma \alpha  n^{1+\frac{1}{\delta^{\numax}-1}} \,, \qquad \CalgoBeta \sim \frac{\beta}{2}n^2 \,,
\end{equation}
with $\delta=\frac{3}{2}$ and $\gamma=\frac{27}{4}\delta^{-\frac{\numax}{\delta^{\numax}-1}}\left( 1-\delta^{-\numax}\right)$. This tree is obtained for $c_\nu = \delta \left( \delta^{-\numax} n\right)^{\frac{\delta^{(\nu-1)}}{2(\delta^{\numax}-1)}}$, $\nu=1, \ldots, \numax$. In a particular two-layers setting one gets $c_1=\left(\frac{3}{2}\right)^{1/5} n^{2/5}$ and $c_2=\left(\frac{3}{2}\right)^{-1/5} n^{3/5}$, which leads to $\CalgoAlpha = \gamma \alpha n^{9/5}$ and $\CalgoBeta= \frac{\beta}{2} n^2  - \frac{\beta}{2} \left(\frac{3}{2}\right)^{\frac{1}{5}}n^{\frac{7}{5}}$, where $\gamma=(\frac{2}{3})^{-2/5}+(\frac{2}{3})^{3/5}\simeq 1.96$.
\end{itemize}
\end{proposition}
\begin{proof}
The details of the proof are given in Appendix~\ref{app:complexities}.
\end{proof}

We have seen that for $q$ prediction points and $n$ observations, a reachable complexity of the algorithm is $O(n^2 q)$, which is less than  $O(n^3)+O(n^2 q)$ for the same prediction with the full model, when $q<n$.

More precisely, we have shown that the choice of the tree structure helps partially reducing the complexity of the algorithm. Indeed, a large tree height $\numax$ largely reduces the complexity $\CalgoAlpha$ of matrix inversions in the algorithm. However, $\CalgoBeta$ cannot be reduced and  one can expect a maximal complexity reduction factor of $\frac{\beta}{2\alpha + \beta}$ when using an optimal tree, compared to the  equilibrated two-layers $\sqrt{n}$-tree. One shall however keep in mind that a lower complexity can lead to larger prediction errors or larger storage footprint.

As a perspective, approximating cross-covariances between aggregated models would allow to reduce $\CalgoBeta$ to the same order as $\CalgoAlpha$, which approaches $O(n)$ when $\numax$ is large. This thus gives perspectives for further large reduction of the global complexity, which are let to future work.

At last, several parts of the algorithm can be computed in parallel execution threads. This is an interesting feature since sub-models computation at any layer can also be distributed.


\section{Parameter estimation}
\label{sec:6estimation}

\paragraph{}
Consider a set of covariance functions $ \{ \sigma^2 k_{\theta} , \sigma^2 \geq 0, \theta \in \Theta\}$  where $k_\theta$ is a correlation function from $D \times D$ into  $[-1,1]$ depending on some parameters $\theta$ such as length-scales. In this section, we address the problem of selecting the value of $\sigma^2$ and $\theta$ from the input observation points in $X$ and the observation vector $f(X)$.
The mean predictor $m_\agg$ depends only on $\theta$ so it will be written $m_{\agg,\theta}$.
The prediction variance is a function of both $\theta$ and $\sigma^2$. Since it is linear in the latter, the prediction variance is written $ \sigma^2 v_{\agg,\theta}$.

\paragraph{}
For $1 \leq i \leq n$, let the leave-one-out mean
$m_{\agg,\theta,-i}(x_i)$
be computed as $m_{\agg,\theta}(x_i)$, but with $X,f(X)$ replaced by $X_{-i},f(X_{-i})$, where $X_{-i}$ is obtained by removing the $i$th line of $X$. Note that the input division $X_1,...,X_p$ is left unchanged, apart from removing $x_i$ when it appears in $X_1,...,X_p$. Similarly, the tree structure $ \set{\setI_i^\nu }$ is left unchanged. We define
$\sigma^2 v_{\agg,\theta,-i}(x_i)$
similarly to $\sigma^2 v_{\agg,\theta}(x_i)$.
\paragraph{}
We estimate $\sigma^2$ and $\theta$ with a two-step leave-one-out procedure similar to that of \cite{bachoc13cross}. We first select $\theta$ as minimizing the leave-one-out mean square error:
\begin{equation}   \label{eq:optim:loo}
\widehat{\theta} \in
\underset{\theta \in \Theta}{\operatorname{argmin}}
\frac{1}{n}
\sum_{i=1}^n \left(  f(x_i) - m_{\agg,\theta,-i}(x_i) \right)^2.
\end{equation}
Second, we set $\sigma^2$ so that the
leave-one-out errors have variance one:
\begin{equation}
\widehat{\sigma}^2 = \frac{1}{n} \sum_{i=1}^n
\frac{
\left(  f(x_i) - m_{\agg,\widehat{\theta },-i}(x_i) \right)^2
}
{
 v_{\agg,\widehat{\theta },-i}(x_i).
}
\end{equation}

We implemented an algorithm that computes, for a given covariance parameter $\theta$, the quantities $m_{\agg,\theta,-i}(x_i)$ and $v_{\agg,\theta,-i}(x_i)$ for $q$ different points $x_i$.
If the proper storage and precomputations are made, the computational cost
is of $O(q n^2)$, which is similar to the cost for predicting at $q$ new locations using the model aggregation procedure presented in this paper.
However using precomputations, the algorithm also has a storage cost of $O(nq)$ which excludes
using $q=n$ in the case where $n$ is large and prevents computing
the right-hand side of Equation~\eqref{eq:optim:loo} exactly.
Finally, one may notice that when $q$ points are chosen uniformly, without replacement,  in the set of all
$n$ points, averaging $q$
leave-one-out mean square error yields an unbiased estimate
of the leave-one-out mean square error,
and can be seen as an approximation of the latter.
We thus propose to solve the optimization problem \eqref{eq:optim:loo} with a stochastic gradient descent algorithm described in Chapter 5 of \cite{bhatnagar13stochastic}. At each step of the gradient descent, the projection of the gradient of \eqref{eq:optim:loo} on a random direction is approximated by a finite difference. The algorithm is as follows.

\vspace{0.5cm}

\begin{algorithm}[H]
\DontPrintSemicolon
\CommentSty{\color{blue}}

\SetKwInOut{Input}{inputs}
\SetKwInOut{Output}{outputs}
\Input{$\theta_0$, initial value of $\theta$ \newline
$(a_i)_{i \in \mathbb{N}}$, sequence of increment terms for the gradient descent \newline
$(\delta_i)_{i \in \mathbb{N}}$, sequence of step sizes for the finite differences \newline
$q$, number of leave-one-out predictions \newline
$n_{iter}$, maximal number of iterations}

\Output{$\hat{\theta}$}
\medskip

\For{$i = 1,...,n_{iter}$}
{
Sample a subset $\mathcal{I}_i$ of $\{1,...,n\}$, uniformly over all the subsets of $\{1,...,n\}$ with cardinality $q$. \;

Sample a $m$-dimensional vector $h_i$ from a $m$-dimensional random vector with independent components, each of them taking the values $1$ and $-1$ with probabilities $1/2$.

Let
\[
\Delta_i = \frac{1}{2 \delta_i}
\left(
\frac{1}{q}
\sum_{j \in \mathcal{I}_i} \left(  f(x_j) - m_{\agg,-j,\theta_{i-1} + \delta_i h_i}(x_j) \right)^2
-
\frac{1}{q}
\sum_{j \in \mathcal{I}_i} \left(  f(x_j) - m_{\agg,-j,\theta_{i-1} - \delta_i h_i}(x_j) \right)^2
\right).
\]

Let $\theta_i = \theta_{i-1} - a_i \Delta_i h_i$.
}

Let $\hat{\theta} = \theta_{n_{iter}}$.
\caption{Stochastic gradient descent}
\label{algo:stochastic:gradient}
\end{algorithm}

\vspace{0.5cm}
\def\Cplusplus{C\raisebox{0.5ex}{\tiny\textbf{++}}}
An implementation in $\textsf{R}$ and \Cplusplus{} of both algorithms~\ref{algoKrigIteratif2} and~\ref{algo:stochastic:gradient} is publicly
available on the website~
\url{http://www.clementchevalier.com/index.php/r-packages}.
In practice, the computation cost of $q$ leave-one-out predictions is the sum of a fixed cost -- involving in particular the computation of the $n^2$
covariances $k_{\theta}(x_i,x_j)$ -- and a marginal cost which is
proportional to $q$. When $n=10,000$, these two summands take comparable values for $q=100$, which is the setting we use in practice. Following the recommendations in \cite{bhatnagar13stochastic}, we set $\delta_i = c/(i+1)^{\gamma}$, with $\gamma = 0.101$. We set $a_i = a / (A+i+1)^{\alpha}$, with $\alpha = 0.602$ (as suggested in \cite{bhatnagar13stochastic}), or $\alpha=0.2$, or a combination of these two values. Typically we run a first gradient descent with $\alpha=0.2$, which termination point serves as starting point for a second gradient descent with $\alpha=0.602$.
Good values of $a$, $c$ and $A$ depend of the application case. In practice, satisfactory results are obtained for $n=10,000$, $d=10$ and $p=100$, with $n_{iter} = 500$, in which case the computation time would be around a few hours on a personal computer \did{with a  mono-threaded implementation}.


\section{Numerical applications}
\label{sec:5discussion}

\subsection{Comparison with other aggregation methods}
\label{sec:51CompSimulatedData}
We now compare the predictions obtained with various methods when aggregating 15 Kriging sub-models based on two observations each. The test functions are samples of a centered Gaussian process over $[0,1]$. The compared models are the nested Kriging model introduced in this article, the full model and other methods developed in the literature:
\vspace{0.1cm}
\\
\textbf{Product of expert (PoE)} \cite{hinton2002training} is based on the assumption that for a given $x$, the predictions of each sub-model correspond to independent random variables. As a consequence, the aggregated predicted density for $Y(x)$ is equal to the product of the sub-models densities : $f_{poe}(y) \propto \prod_{i=1}^p f_i(y)$ where $f_i$ is the predicted density of $Y(x)$ according to the $i$th sub-model. The PoE corresponds to the normal model developed in~\cite{winkler1981combining}, in the case of independent experts, when the considered covariance matrix is diagonal (see \eg section 3.2 in the previously cited article and \cite{van2015optimally}). Some extensions of this method to consensus Monte-Carlo sampling can be found in~\cite{scott2016bayes}.
\vspace{0.1cm}\\
\textbf{Generalized product of expert (GPoE)}. As discussed in~\cite{deisenroth2015}, a major drawback of Kriging based PoE is that the prediction variance of the aggregated model decreases when the number of sub-models increases even in regions with no observation points. \cite{caoGPoE} introduced a variant called generalized product of expert where a weighting term is added to overcome this issue. The prediction is then given by
\begin{equation} \label{eq:def:GPoE}
	f_{gpoe}(y) \propto \prod_{i=1}^p (f_i(y))^{\beta_i}.
\end{equation}
For this benchmark, the parameters $\beta_i$ will be set to $1/p$ as recommended in \cite{deisenroth2015}. Notice that GPoE corresponds exactly to what consensus literature refers to \textit{logarithmic opinion pool}, see \eg Eq.(3.11) in \cite{genest1986combining}.
\vspace{0.1cm}\\
\textbf{Bayesian Committee Machine (BCM)} has been introduced in~\cite{trespBCM} to aggregate Kriging sub-models. It is based on the assumption of conditional independence of the sub-models given the process values at prediction points. The predicted aggregated density is given by
\begin{equation}
  f_{bcm}(y) \propto \frac{\prod_{i=1}^p f_i(y)}{f_{Y}(y)^{p-1}}.
\end{equation}
\vspace{0.1cm}\\
\textbf{Robust Bayesian committee machine (RBCM)} has been introduced in~\cite{deisenroth2015} to correct some supposed flaws from BCM aggregations in the case where there are only few observations in each \did{sub-model}. The predicted aggregated density is given by
\begin{equation}
	f_{rbcm}(y) \propto \frac{\prod_{i=1}^p (f_i(y))^{\beta_i}}{(f_{Y}(y))^{-1 + \sum_i \beta_i}},
\end{equation}
where $\beta_i = \frac{1}{2} [\log(\Var{\Y(x)} ) - \log(v_i(x))] $ with $v_i(x)$ the predicted variance of the $i^{th}$ sub-model at $x$.
\paragraph{}
One advantage of these aggregation methods is their very low complexity.
However, GPoE, BCM and RBCM can be proven to be inconsistent~\cite{bachoc2017}. For the sake of comparison, we add two other methods to the benchmark:
\vspace{0.1cm}\\
\textbf{Smallest prediction variance (SPV)}. For a given prediction point $x$, the aggregation returns the prediction of the sub-model with the lowest prediction variance:
\begin{equation}
	f_{spv}(y) = f_k(y) \text{\qquad with } k = \mathop{\operatorname{argmin}}_{i \in \set{1,\dots,p}} v_i(x).
\end{equation}
\vspace{0.1cm}\\
\textbf{Nearest Neighbors (NN)}. This is not strictly speaking an aggregation method: For a given prediction point $x$, it consists in a kriging predictor based only on the $k$ observations points that are the closest to $x$.
\paragraph{}
These last two methods do not suffer from the inconsistency discussed for the above methods, but they provide discontinuous predictions. This can be an issue for example if the (aggregated) model is used to perform some optimization tasks.

\paragraph{}
The test functions are given by samples over $[0,1]$ of a centered Gaussian process $Y$ with a Mat\'ern kernel of smoothness 5/2. The variance and length-scale parameters of the latter are fixed to $\sigma^2 = 1$ and $\theta = 0.05$. The vector of observation points $X$ consists of 30 random points uniformly distributed on $[0,1]$ and we consider in this example the aggregation of 15 sub-models based on two points each. Assuming that the observations points are ordered ($x_1 \leq ... \leq x_n $), each sub-model is trained with two consecutive observations points : $\agg_1 = \{1,2 \}$,\ \dots,\ $\agg_{15} = \{29,30 \}$. The variance and length-scale parameters of the sub-models are equal to the values used to generate the process samples.

\paragraph{}
First of all, we will focus on the aggregated models obtained with the different methods for a given sample path and design of experiments $X$ before looking at the distribution of various criteria when replicating the experiment. Figure~\ref{fig:comparison} shows the aggregated models for the aggregation methods described above. On this example, PoE and GPoE appear respectively to be over- and under-confident in their predictions and show a mean prediction that tends too quickly to zero as the prediction point moves away from the observation points. On the other hand, the predictions from other methods seem more reliable and the best approximation is obtained with the proposed nested estimation approach.
\begin{figure}[htp]
  \centering
  \subfloat[PoE]{\label{fig:compPOE}\includegraphics[width=.36\textwidth]{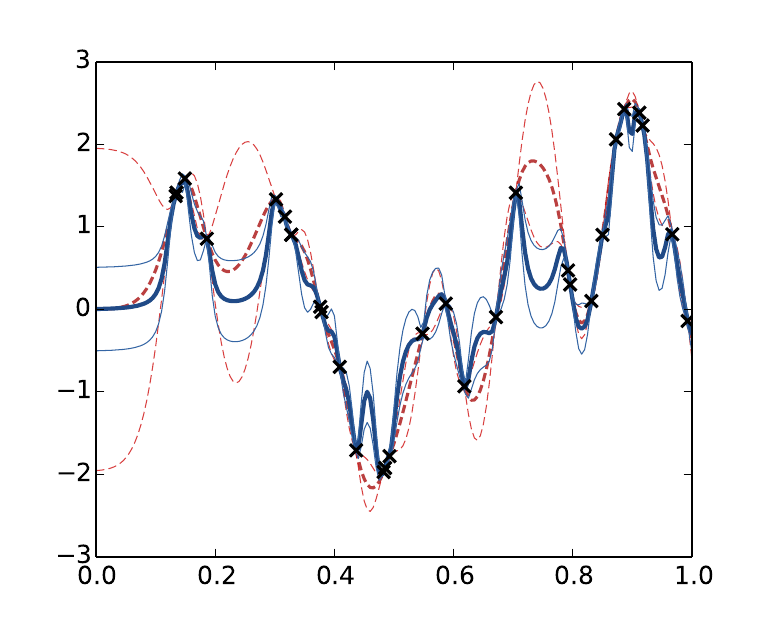}} \qquad
  \subfloat[GPoE]{\label{fig:compGPoE}\includegraphics[width=.36\textwidth]{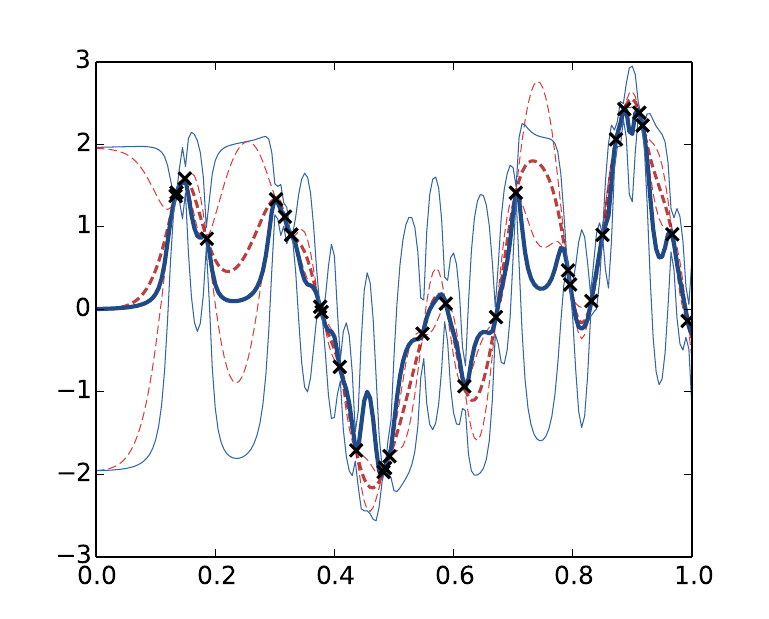}}\\
  \subfloat[BCM]{\label{fig:compBCM}\includegraphics[width=.36\textwidth]{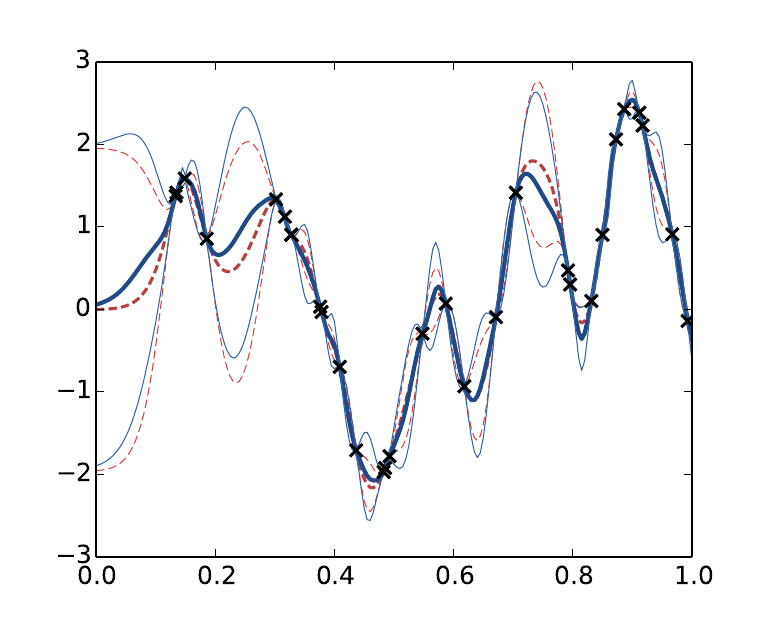}} \qquad
  \subfloat[RBCM]{\label{fig:compRBCM}\includegraphics[width=.36\textwidth]{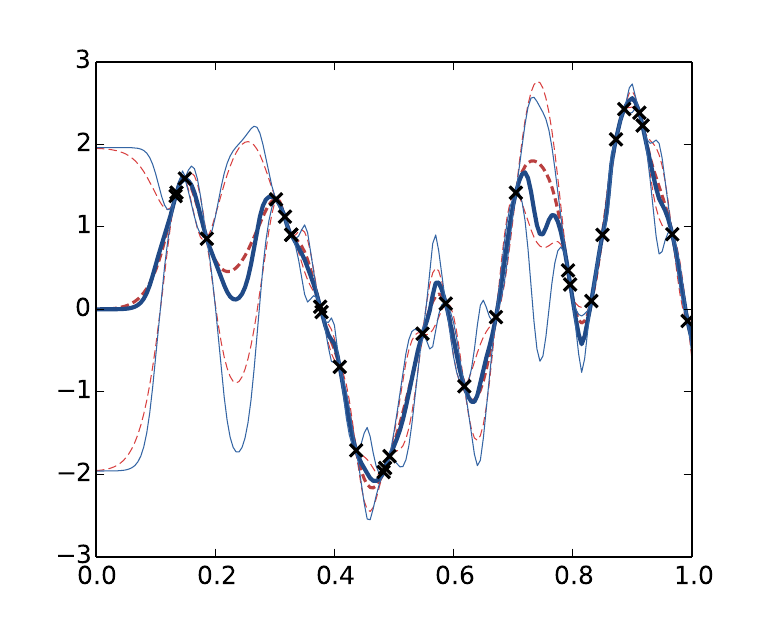}}\\
	\subfloat[SPV]{\label{fig:compSPV}\includegraphics[width=.36\textwidth]{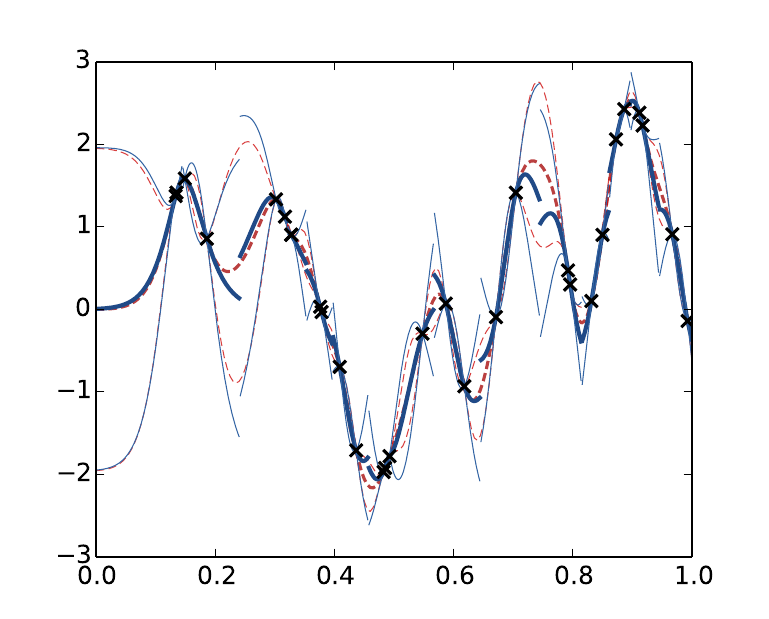}} \qquad
  \subfloat[NN]{\label{fig:compNN}\includegraphics[width=.36\textwidth]{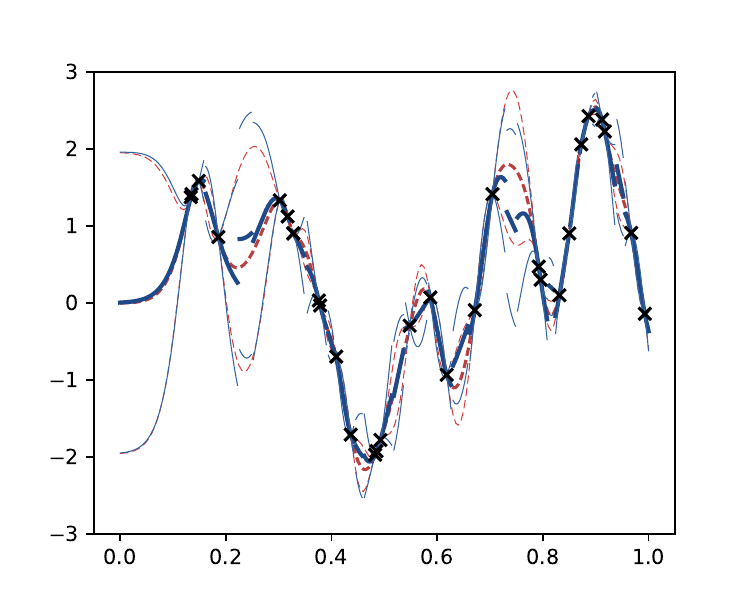}}\\
	\subfloat[nested GPR]{\label{fig:compNest}\includegraphics[width=.36\textwidth]{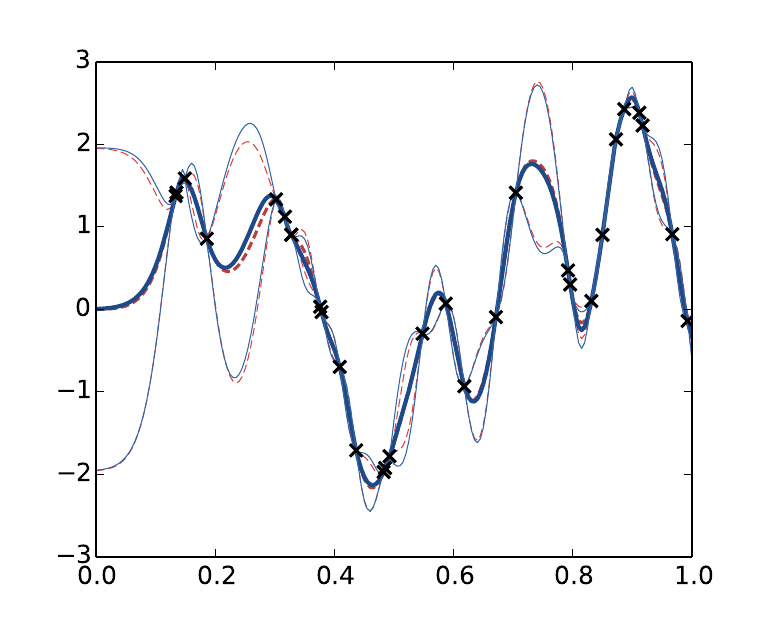}} \qquad
  \subfloat[full]{\label{fig:compfull}\includegraphics[width=.36\textwidth]{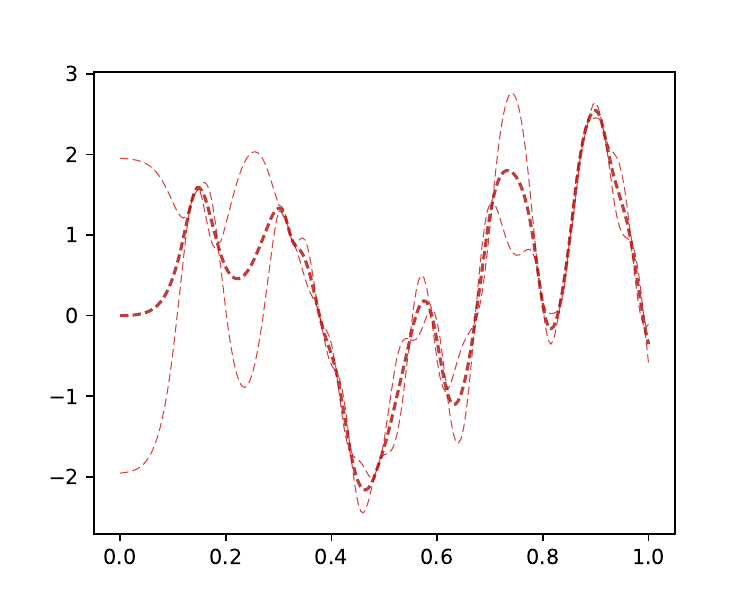}}
  \caption{Comparison of various aggregation methods. The solid lines correspond to aggregated models (mean and 95\% prediction intervals) and the dashed lines indicate the full model predictions (mean and 95\% prediction intervals).}
  \label{fig:comparison}
\end{figure}

\paragraph{}
This can be confirmed by replicating 50 times the experiment by sampling independently the observation points and the test function. We consider three criteria to quantify the distance between the aggregated model and the full model: the mean square error (MSE) to assess the accuracy of the aggregated mean, the mean variance error MVE for the accuracy of the predicted variance and the mean negative log probability (MNLP) \cite{Rasmussen2006} to quantify the overall distribution fit. Let $m, v$ (resp. $m_\full, v_\full$) denote the mean and variance of the model to be tested (resp. the full model) and let ${X}_t$ be the vector of test points. These criteria are defined as:
\begin{equation}
	\accoladesplit{
		MSE(m,m_\full,{X}_t) & = \frac{1}{n_t} \sum_{i=1}^{n_t} ( m(x_{t,i}) - m_\full(x_{t,i}) )^2 \virguleacc \\
		MVE(v,v_\full,{X}_t) &= \frac{1}{n_t} \sum_{i=1}^{n_t}  ( v(x_{t,i}) - v_\full(x_{t,i}) ) \virguleacc \\
		MNLP(m,v,f,{X}_t) & = \frac{1}{n_t} \sum_{i=1}^{n_t} \left( \frac{1}{2} \log(  2 \pi v(x_{t,i}) ) + \frac{ ( m(x_{t,i}) - f(x_{t,i}) )^2} { 2 v(x_{t,i}) } \right) \pointacc \\
	}
	\label{eq:crit}
\end{equation}
Figure~\ref{fig:comparisonBoxplots} shows the boxplots of these criteria for 50 replications of the experiments. It appears that the proposed approach gives the best approximation of the full model for the three considered criteria.
\begin{figure}[htp]
  \centering
	\subfloat[MNLP]{\label{fig:compbxplt3}\includegraphics[width=.32\textwidth]{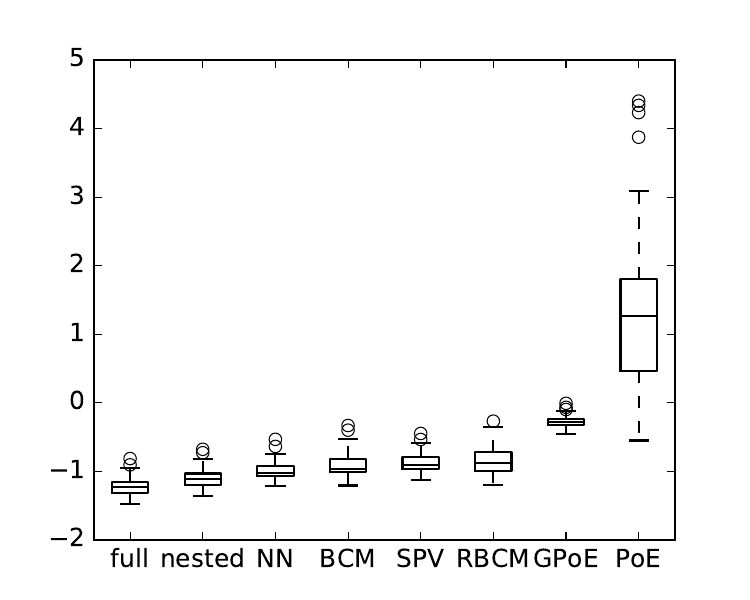}}
  \subfloat[MSE]{\label{fig:compbxplt1}\includegraphics[width=.32\textwidth]{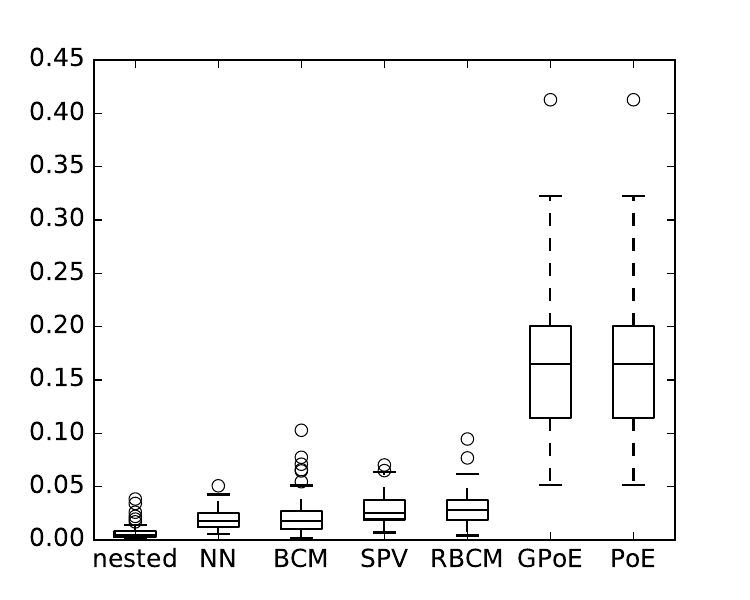}}
  \subfloat[MVE]{\label{fig:compbxplt2}\includegraphics[width=.32\textwidth]{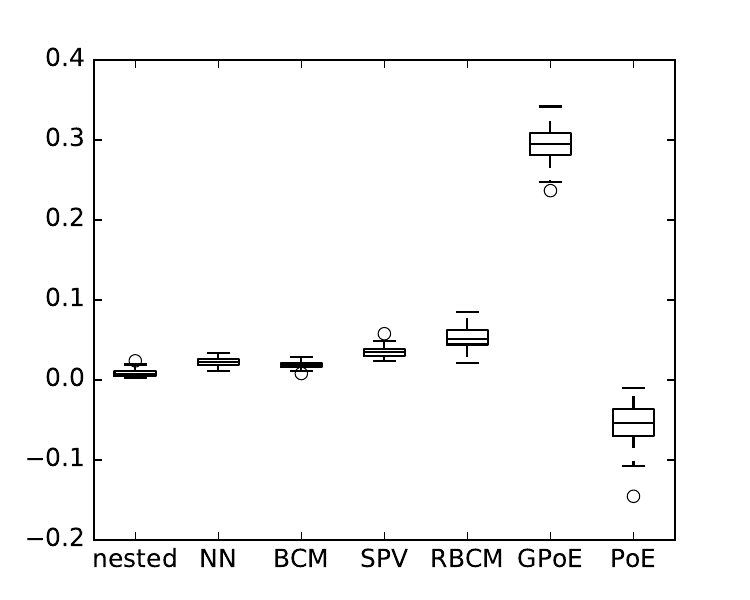}}
  \caption{Quality assessment of the aggregated models for 50 test functions. Each test function is a sample from a Gaussian process and in each case 30 observation points are sampled uniformly on $[0,1]$. The test points vector ${X}_t$ consists of 101 points regularly spaced from $x_{t,1}=0$ to $x_{t,101}=1$.}
  \label{fig:comparisonBoxplots}
\end{figure}

\subsection{Application to a high dimensional input space}
\label{sec:applicationHighDim}
\paragraph{}
We replicate in this section the same experiment as in the previous one but for test functions defined over the unit cube in 100 dimensions. We set the number of training points to $\did{10,000}$ and we generate $100$ sub-models based on a k-means clustering of the input points. This \did{implies} that the average number of points per cluster is 100 so we use the 100 closest observation points in the nearest neighbors method. As previously the test functions are random samples of a centered Gaussian process with squared exponential covariance and we consider two length-scale values to study this parameter influence on the methods to compare: a ``short'' length-scale $\theta=2$ for which the full model captures about 50\% of the prior variance and a ``large'' length-scale $\theta=5$ for which the full model can explain 99\% of the prior variance.

\paragraph{}
The results of the experiment are displayed in Figures~\ref{fig:Applidim100theta2} and ~\ref{fig:Applidim100theta5}. The first striking observation is that BCM, RBCM and PoE underestimate the variance (since MVE\did{s} are negative) \did{and} lead to highly overconfident models. Regarding the other approaches, NN, SPV and GPoE seem to provide similar global accuracy although it can be noted that the Nearest Neighbor method mean predictions are inaccurate for large length-scales. This can be explained by the set of influential neighbors being larger than 100 for such values of the length-scales. Finally, it can be seen that the proposed nested method is the one providing the best approximation of the full model. \did{Compared} to the other methods its mean is more accurate and it provides prediction intervals that are smaller than NN, SPV and GPoE while being realistic as shown by the MNLP. This is especially true for large length-scales for which the proposed approach is even more competitive.

\begin{figure}[htp]
  \centering
	\subfloat[MNLP (log scale)]{\label{fig:d10t2MNLP}\includegraphics[width=.32\textwidth]{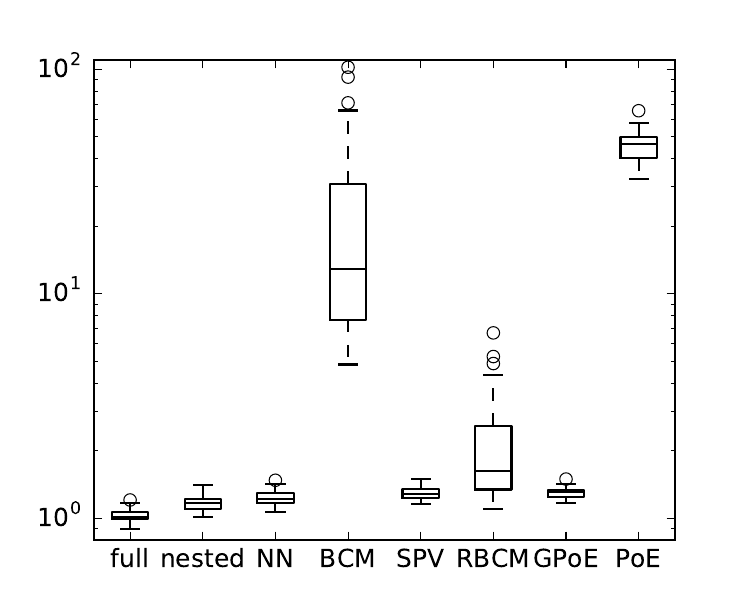}}
  \subfloat[MSE (log scale) ]{\label{fig:d100t2MSE}\includegraphics[width=.32\textwidth]{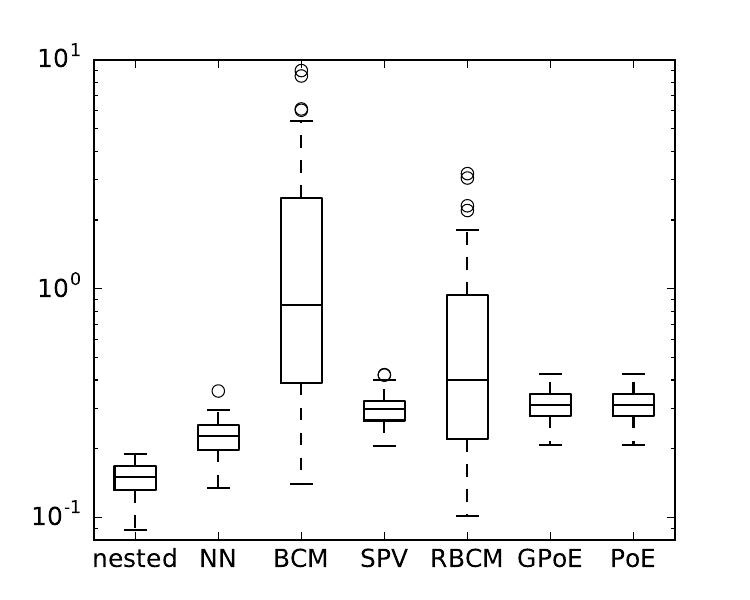}}
  \subfloat[MVE]{\label{fig:d10t2MVE}\includegraphics[width=.32\textwidth]{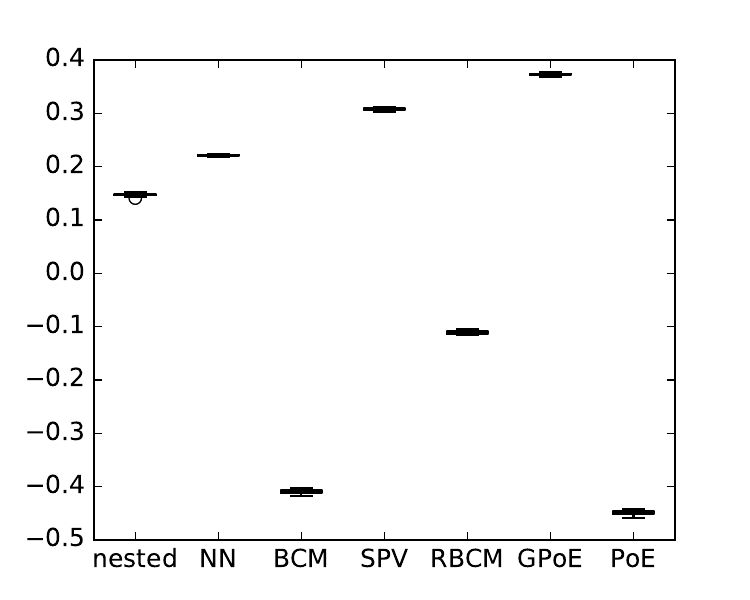}}
  \caption{Quality assessment of various aggregation methods. The test functions are given by 50 samples from a centered Gaussian process over $[0,1]^{100}$ with squared exponential kernel, unit variance and length-scale $\theta=2$. The models are built using $\did{10,000}$ observations points drawn uniformly in the input space. These input points are gathered into 100 groups using k-means in order to build the sub-models. The test point locations are obtained by sampling uniformly 100 points in the input space.}
  \label{fig:Applidim100theta2}
\end{figure}

\begin{figure}[htp]
  \centering
	\subfloat[MNLP]{\label{fig:d10t5MNLP}\includegraphics[width=.32\textwidth]{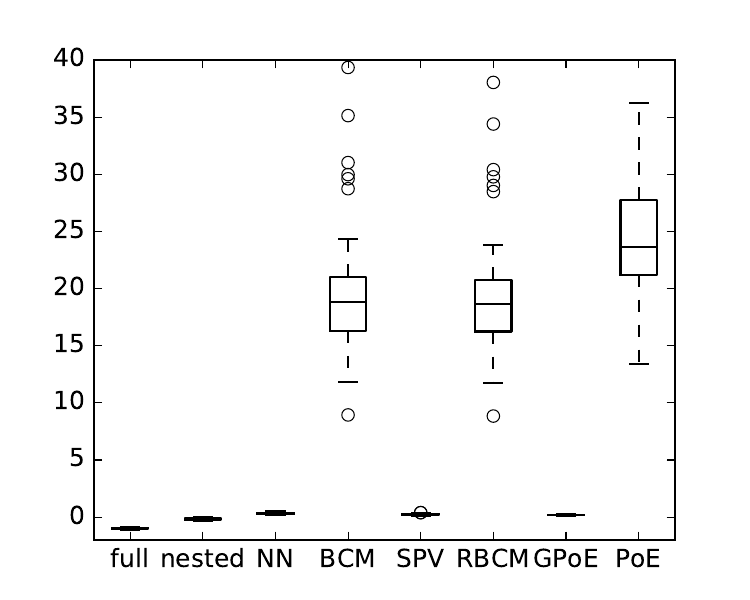}}
  \subfloat[MSE]{\label{fig:d100t5MSE}\includegraphics[width=.32\textwidth]{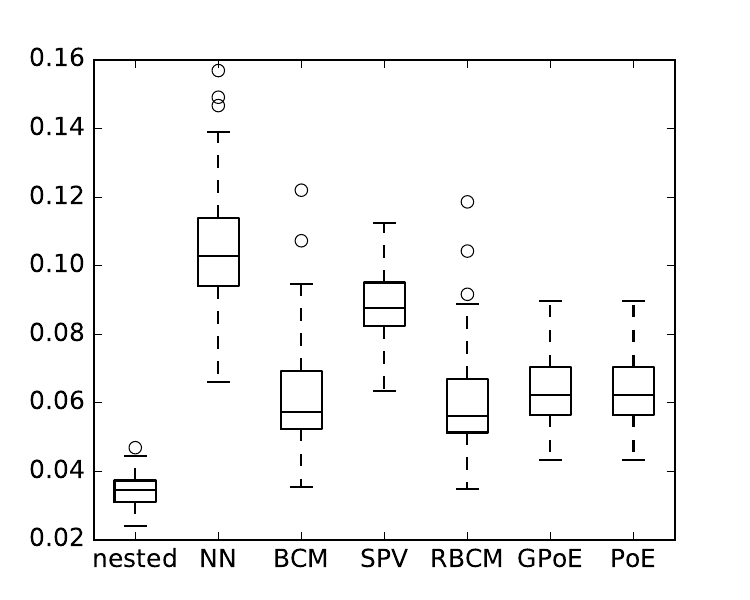}}
  \subfloat[MVE]{\label{fig:d10t5MVE}\includegraphics[width=.32\textwidth]{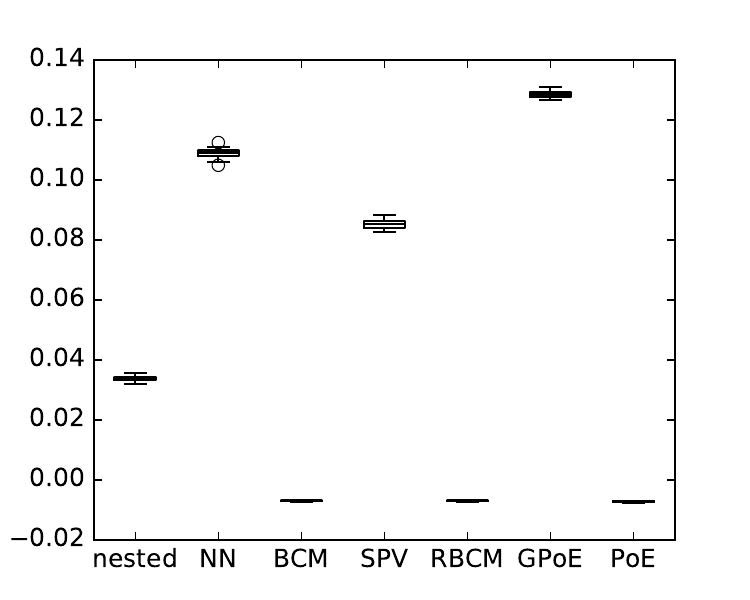}}
  \caption{Same settings as in Figure~\ref{fig:Applidim100theta2}, but the length-scale of the test functions is 5.}
  \label{fig:Applidim100theta5}
\end{figure}

\subsection{Application to a large dataset}
\label{sec:7applicationHartman}
In this section, we analyze the performance of the proposed method on a test function with one million observations.\\

We use two different test functions. The first one is the \textrm{Hartman6} test function \did{in} dimension $d=6$ (available for example in the \textrm{DiceKriging} \textrm{R} package \cite{roustant12dice}). The second one in the dimension $d=18$ is called here \textrm{Hartman18} test function: it is simply the sum of three \textrm{Hartman6} test functions, each acting on 6  separated parameters: $\mathrm{Hartman_{18}}(x_{1:18})=\mathrm{Hartman_{6}}(x_{1:6})+\mathrm{Hartman_{6}}(x_{7:12})+\mathrm{Hartman_{6}}(x_{13:18})$.\\

For the Hartman6 test function, the covariance parameters of a squared exponential kernel have been estimated once on a subset of points. We give here the obtained length-scales, so that the results can be easily reproduced: $(0.262, 0.435, 0.423, 0.348, 0.314, 0.299)$.
For the Hartman18 test function, we use two different sets of covariance parameters: the first one is slightly misspecified since it is given by the one estimated for Hartman6 repeated three times, the second one is estimated with usual MLE on a subset of 2000 points. Although the model could be improved with a refined estimation of the length-scales, this is sufficient to compare the different methods. The variance parameter has no influence on this comparison since we only consider here the performance of the mean predictors.\\

We consider in this example $n=1,000,000$ design points and $q=100$ predictions points. Several methods are considered:
\begin{itemize}
\item the \textit{Kriging} predictor refers to a simple Kriging predictor based on a random sample (without replacement) of $1000$ points taken among the initial points. It is mainly computed in order to give an order of the reachable error magnitude for a reasonable learning of the test function, and to see if refined methods really improve the performance of the prediction.
\item The \textit{Neighbor} predictor refers to a simple kriging predictor which gives, for each prediction point, the prediction based on its nearest neighbors in the design matrix $X$. \textit{$Near_{100}$} refers to a predictor based on $100$ nearest neighbors, \textit{$Near_{1000}$} refers to a predictor based on $1000$ nearest neighbors.
\item The \textit{Nested} method refers to the proposed method in this paper. For one million points, we have chosen a tree structure corresponding to $N=1000$ groups of points, each group being obtained using \textit{kmeans} clustering algorithm. Two variants are considered: \textit{Nested} refers to a clustering that is built directly without considering the locations of the prediction points. \textit{Nested+} refers to a clustering that is built using these locations (i.e. first $q=100$ clusters are built around each prediction points, without overlapping, and $N-q=900$ clusters are built on residual design points). In all cases, depending on the location of design points, each cluster size typically vary between $800$ and $1200$.
\end{itemize}

For each run, we draw uniformly a new design matrix $X$, a new vector of predictions points $x$, and we analyze the performance of the predictors. To this aim, the predictions are compared to the true chosen test function, and we collect for each run one mean of  errors over all prediction points. We reproduce the whole experiment on 10 runs. The results are gathered in the boxplots of Figure~\ref{fig:boxplotHartman6} and Figure~\ref{fig:boxplotHartman18}.\\

\begin{figure}
\centering
\subfloat[Hartman 6]{\includegraphics[width=0.48\linewidth]{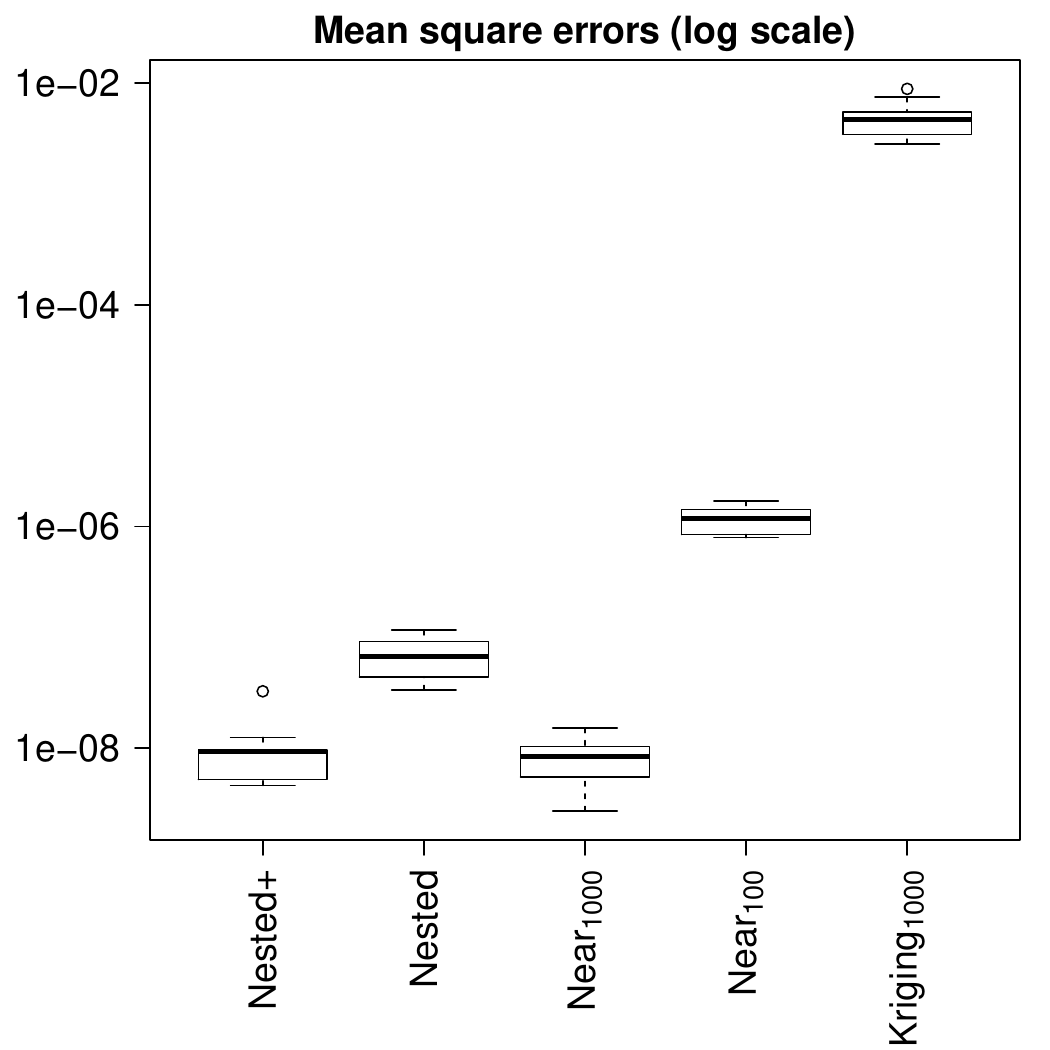}}
  \caption{\did{Boxplot} of the mean square errors (in log scale), for the Nested procedure and its variant, the nearest neighbors procedure based on 1000 neighbors or 100 neighbors, and for the Kriging method based on a random sample of 1000 points. \did{We used } one million input points, one hundred prediction points  \did{and} Hartman6 test function.}
  \label{fig:boxplotHartman6}
\end{figure}

\begin{figure}
\centering
\subfloat[Hartman 18, concatenated parameters]{\includegraphics[width=0.48\linewidth]{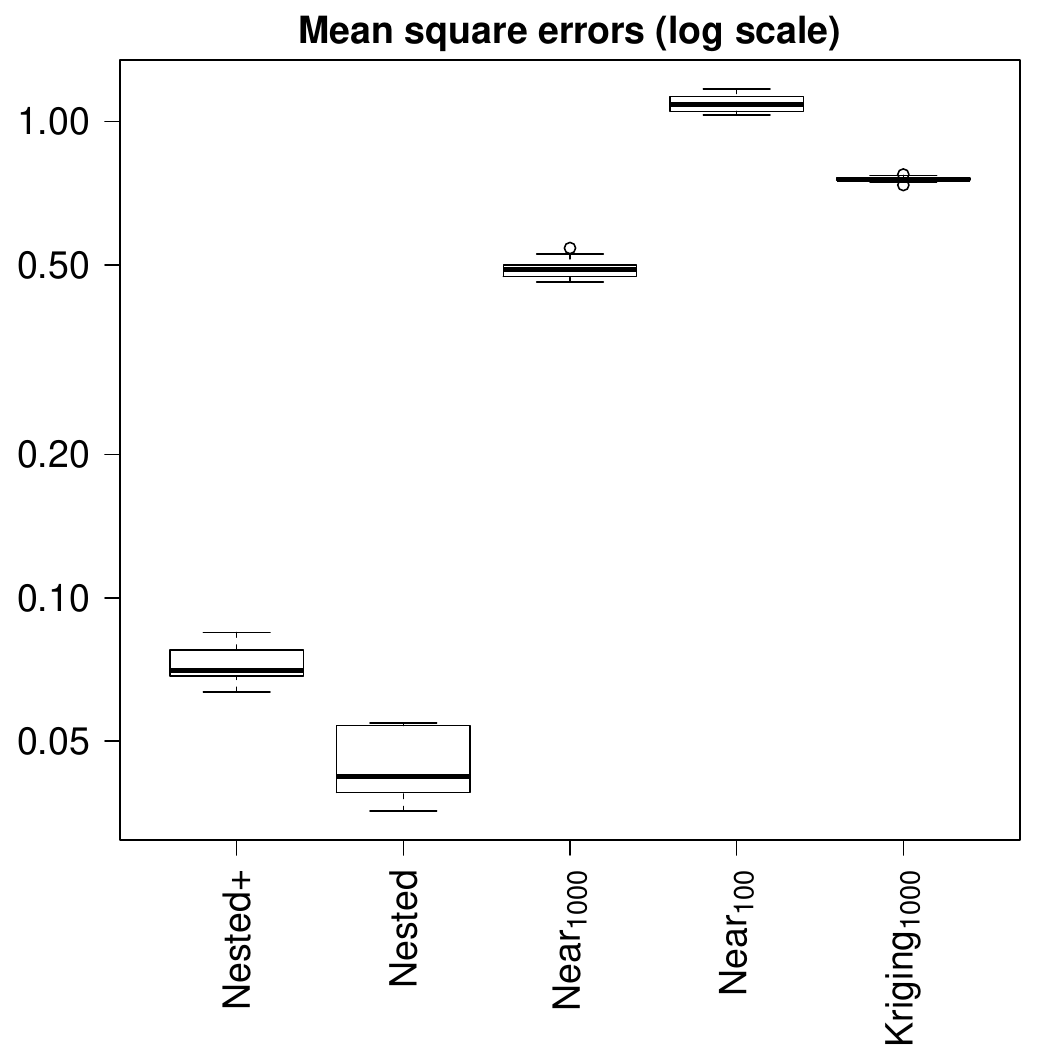}}
  \subfloat[Hartman 18, estimated parameters]{\includegraphics[width=0.48\linewidth]{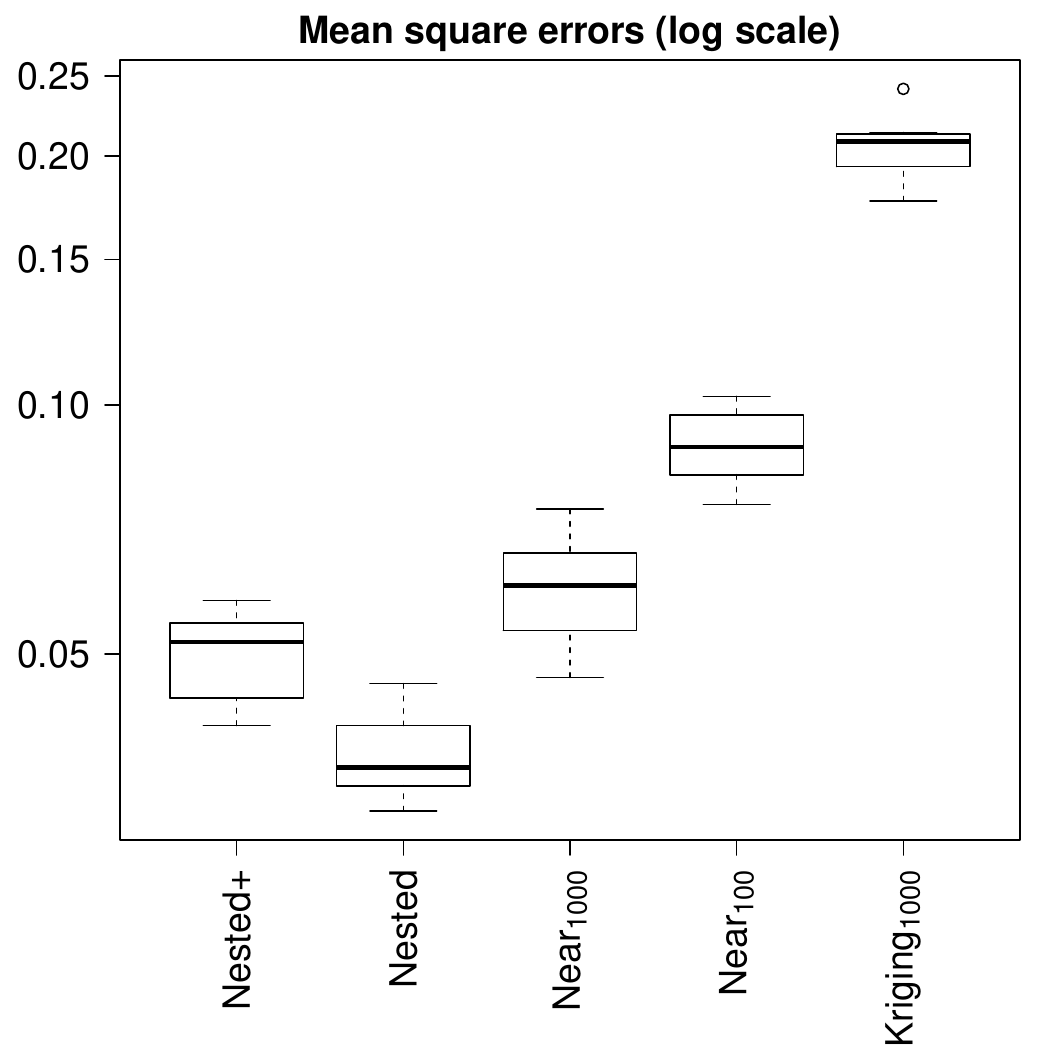}}
  \caption{\did{Boxplots} of mean square errors (in log scale), for the Nested procedure and its variant, the nearest neighbors procedure based on 1000 neighbors or 100 neighbors, and for the Kriging method based on a random sample of 1000 points. \did{We used } one million input points, one hundred prediction points, Hartman18 test function \did{and either} concatenated parameters from Hartman6 (left panel) or estimated parameters (right panel).}
  \label{fig:boxplotHartman18}
\end{figure}


As announced, the Kriging predictor based on a random sample of 1000 points is mainly given to get an order of the errors magnitude with a reasonable learning of the function, but clearly its complexity is lower, and it does not reach the precision level of other competitors: other methods outperform results based on a random sample of experiments.\\
The nearest neighbor is a clear improvement over the randomly selected training points and the computation scheme is relatively simple. However it generates non-continuous mean and variance predictions, so the practical interest of such \did{a} model (e.g. to perform optimization) may be limited. While it performs very well in small dimension, one can see on Figure~\ref{fig:boxplotHartman18} that it becomes less attractive in higher dimension since it would require to further increase the number of neighbors to provide competitive predictions.
One can also see that nearest neighbors are also quite sensitive to the estimation of the parameters.\\
The nested procedure has a greater complexity\did{:} each run takes around \did{$50$ minutes} on a modern computer with 16 threads \did{($1,000,000$ observations, $100$ prediction points, Hartman18 test function, $1000$ groups and possible multithreading scalability improvements). The execution time for $100,000$ observations and unchanged other settings is around $30$ seconds}. \did{Our procedure} remains here accurate with concatenated parameters, \did{contrarily} to other methods. It leads to a very good accuracy and some theoretical advantages previously presented.
In relatively small dimension, as in Figure~\ref{fig:boxplotHartman6}, depending on the length-scales of the underlying process, the closest neighbors may be sufficient to explain the local shape of the response. In this case, a judicious choice of the tree structure may improve the accuracy of the nested method, which is comparable to the one of the $1000$ closest neighbors. In larger dimension, when local information is not sufficient, the choice of the tree structure has a lower impact, and the refinements of the nested method make sense, as they lead to a better accuracy than considered local neighbors methods.\\

Finally, despite greater complexity, the proposed method is still tractable with one million of observations. It leads to a better accuracy, especially in high dimension. In small dimension and when possible, it can be useful to build the tree structure by using the location of the prediction points, to take the best of both closest neighbors and nested methods. At last, the proposed nested method makes an intensive use of cross-covariances between groups and can surely be improved by using a better estimation of these parameters, or by a transformation of the inputs or the outputs that would make the assumptions more reliable.

\subsection{Application to an industrial case study}
\label{sec:7application}

\paragraph{}
We consider in this section experimental data on the behavior of a steel test piece subject to cycles of tension-compression. During these cycles, the evolution of the tensile strain in the test piece is monitored over time using two methods: by performing the actual physical experiment and by a numerical simulator based on a Chaboche constitutive equation \cite{lemaitre1994mechanics}. The quantity of interest is the misfit between these two experiments. A test piece is described by $6$ scalar variables $(E,C_1,C_2,\gamma_1^0, \gamma_2^0, r)$, where $E$ is a logarithm transform of the Young's modulus, $C_1$, $C_2$, $\gamma_1^0$ and $\gamma_2^0$ are parameters related to the kinematic hardening and $r$ is the radius of the plastic surface at the stabilized state. The set of admissible inputs is denoted by $D \subset \R^6$.

Hereafter, we focus on modeling the function $f: D \rightarrow \R$ that returns the logarithm of the $L^2$ norm of the difference between the curve from the actual experiment and the one from the simulator.

\paragraph{}
In total, we have at our disposal a set of $10,000$ observations $[X,f(X)]$, from which we randomly extract a learning set $[X_{l},f (X_l)]$ of $n=\did{9000}$ observations and assign the $n_t = \did{1000}$ remaining observations to a test set $[X_t,f (X_t)]$.

\paragraph{}
We compare the predictions of $f(X_t)$ obtained from the SPV, PoE, GPoE1, GPoE2, BCM and RBCM aggregation procedures described in Section~\ref{sec:51CompSimulatedData} with our nested aggregation procedure. GPoE1 corresponds to \eqref{eq:def:GPoE} with $\beta_i = \frac{1}{2} [\log(\Var{\Y(x)} ) - \log(v_i(x))] $ \cite{caoGPoE} and GPoE2 corresponds to \eqref{eq:def:GPoE} with $\beta_i = 1/p$ \cite{deisenroth2015}. For all these methods, we consider an aggregation tree of height $\numax = 2$ (once sub-models have been evaluated at layer 1, they are all directly aggregated into one value at layer 2), so that $p$ Gaussian process models are directly aggregated. The $p$ subsamples form a partition of $[X_l,f(X_l)]$, which is obtained using the $k$-means clustering algorithm.

\paragraph{}
Three covariance functions have been considered for the sub-models: (tensorized) exponential, Mat\'ern $3/2$ and Mat\'ern $5/2$ (see \cite{Rasmussen2006,roustant12dice} for the definition of these functions). For all studied methods, the Mat\'ern $5/2$ covariance seemed to be the most appropriate to the problem at hand since we obtained overall more accurate results. The results presented hereafter thus focus on this Mat\'ern $5/2$ covariance family. Its parameters are estimated with two different techniques depending on the aggregation method: for the methods from the literature and SPV, we follow the recommended procedure which consists in maximizing the sum of the log likelihoods over the $p$ subsamples of $[X_l,f(X_l)]$ (see \cite{deisenroth2015}). For the proposed nested aggregation, we carry out the stochastic-gradient based estimation method described in Section~\ref{sec:6estimation}, with starting points set to the maximizer of the sum of the log likelihoods.

\paragraph{}
To assess the quality of a model with predicted mean $m$ and variance $v$, we compute three quality criteria using the test set: MSE and MNLP as per Eq.~\ref{eq:crit} which are small for a good model, and the mean normalized square error (MNSE)
\[
MNSE(m,v,f,{X}_t) = \frac{1}{n_t} \sum_{i=1}^{n_t} \frac{ ( m(x_{t,i}) - f(x_{t,i}) )^2}
{ v(x_{t,i}) },
\]
which should be close to $1$.

\paragraph{}
The prediction results for a given learning and training test set are given in Table~\ref{table:p20:kmeans:matern5:2} for the aggregation of $p=20$ sub-models and in Table~\ref{table:p90:kmeans:matern5:2} for $p=90$. It can be seen that in both cases the proposed method outperforms the other aggregation methods for the MSE and MNLP quality criteria. The MSE has the same order of magnitude for the SPV and our aggregation method,
where the prediction errors are small compared to the empirical variance of the test outputs $f(x_{t,i})$, $i=1,...,n_t$, which is approximately equal to $0.81$.
In \did{contrast},  the MSE can be significantly larger for all the other aggregation procedures. For the PoE, GPoE1, BCM and RBCM aggregation techniques,
the values of MNSE are orders of magnitude greater than the target value one, which indicates that the aggregated models are highly overconfident. The GPoE2 aggregation technique is also overconfident when $p=90$, where its MNSE is equal to 5.16.
 The SPV and our aggregation methods provide appropriate predictive variances, and our method provides the best combination of predictions and predictive variances, according to the MNLP criterion.

\begin{table}
\begin{center}
\begin{tabular}{ c | c   c  c  c  c c   c   }
\hline
     & SPV       & PoE       & GPoE1     & GPoE2     & BCM       & RBCM      & Nested  \\
MSE  & $0.00416$ &  $0.0662$ & $0.0033$ & $0.0662$  & $0.604$ & $0.0625$ & $\mathbf{0.00321}$ \\
MNSE & $1.27$ & $20.00$ & $4.55$ & $\mathbf{1.00}$ & $219$ & $60.8$ & $0.846$ \\
MNLP & $-1.86$ & $7.25$ & $-0.949$ & $-0.765$ & $107$ & $27.2$ & $\mathbf{-1.97}$ \\
\hline
\end{tabular}
\end{center}
\caption{Prediction performances of the aggregation of $p=20$ sub-models for the steel piece constraints cycles data set. The investigated prediction performance criteria are the mean square error (MSE) which should be minimal, mean normalized square error (MNSE) which should be close to $1$ and mean negative log probability (MNLP) which should be small. Bold figures indicate each line's best performing aggregation method.}
\label{table:p20:kmeans:matern5:2}
\end{table}

\begin{table}
\begin{center}
\begin{tabular}{ c | c   c   c  c   c   c  c  }
\hline
     & SPV       & PoE       & GPoE1     & GPoE2           & BCM       & RBCM      & Nested \\
MSE  & $0.00556$  & $0.811$ & $0.0244$ & $0.811$  & $1.84$  & $0.121$ & $\mathbf{0.00418 }$ \\
MNSE &  $1.20$ & $465$ & $34.2$ & $5.16$ & $980$ & $148$    & $\mathbf{0.84700}$ \\
MNLP & $-1.55$ & $230$ & $14.1$ & $2.13$ & $487$  & $71$    & $\mathbf{-1.7}$ \\
\hline
\end{tabular}
\end{center}
\caption{Prediction performances of the aggregation of $p=90$ sub-models for the steel piece constraints cycles data set. All other settings are the same as in Table \ref{table:p20:kmeans:matern5:2}.}
\label{table:p90:kmeans:matern5:2}
\end{table}

\paragraph{}
Tables \ref{table:p20:kmeans:matern5:2} and \ref{table:p90:kmeans:matern5:2} also show that aggregating $p=20$ sub-models gives more accurate models than aggregating $p=90$ sub-models. This suggests that it is a good practice to aggregate few sub-models based on many points instead of aggregating many sub-models based on few points. Although this would require further testing to be confirmed, it is not surprising since aggregation methods rely on some independence assumptions that are not often met in practice.

\paragraph{}
Tables~\ref{table:p20:random:matern5:2} and \ref{table:p90:random:matern5:2} show the values of the quality criteria when the subsamples used for the $p=20$ or $p=90$ sub-models are randomly generated into the learning set. They can thus be compared to Tables~\ref{table:p20:kmeans:matern5:2} and \ref{table:p90:kmeans:matern5:2} to study the influence of the choice of the support points of the sub-models\did{:} the criteria values are overall better in Tables~\ref{table:p20:kmeans:matern5:2} and \ref{table:p90:kmeans:matern5:2} so using $k$-means is beneficial for the aggregation procedures.
In addition, our proposed aggregation technique becomes better in comparison to the other methods, and specifically to SPV, when the subsamples are randomly generated.

\begin{table}
\begin{center}
\begin{tabular}{ c | c   c   c  c   c   c  c  }
\hline
     & SPV      & PoE      & GPoE1     & GPoE2    & BCM       & RBCM     & Nested \\
MSE  & $0.0086$ & $0.00763$ & $0.00704$ & $0.00763$ &  $0.338$  & $0.274$ & $\boldsymbol{ 0.00539}$ \\
MNSE & $1.21$ & $9.38$ & $16.6$ & $0.469$ & $178$ & $268$   & $\boldsymbol{0.864}$ \\
MNLP & $-1.25$ & $1.75$ & $5.03$ & $-1.21$ & $86.2$ & $130$   & $\boldsymbol{-1.5}$ \\
\hline
\end{tabular}
\end{center}
\caption{Same settings as in Table \ref{table:p20:kmeans:matern5:2} but when the subsamples are randomly selected.}
\label{table:p20:random:matern5:2}
\end{table}

\begin{table}
\begin{center}
\begin{tabular}{ c | c   c   c  c   c   c  c  }
\hline
     & SPV      & PoE      & GPoE1     & GPoE2    & BCM       & RBCM     & Nested \\
MSE  & $0.0182$ & $0.0293$ & $0.0246$ & $0.0293$  & $0.977$  & $0.686$ & $\boldsymbol{ 0.00575}$ \\
MNSE & $1.29$ & $42.5$ & $57.2$ & $0.473$ & $852$ & $988$  & $\boldsymbol{0.867}$ \\
MNLP & $-0.804$ & $18.3$ & $25.3$ & $-0.517$ & $423$ & $491$   & $\boldsymbol{-1.37}$ \\
\hline
\end{tabular}
\end{center}
\caption{Same settings as in Table \ref{table:p20:kmeans:matern5:2} but with $p=90$ sub-models and where the subsamples are randomly selected.}
\label{table:p90:random:matern5:2}
\end{table}

\paragraph{}
All previous results have been obtained for a given random choice of the learning and test sets. We now replicate the procedure 20 times, with the same settings as in Tables~\ref{table:p20:kmeans:matern5:2} ($p=20$; subsamples obtained from the $k$-means algorithm; Mat\'ern $5/2$ covariance function)
and \ref{table:p90:random:matern5:2} ($p=90$; subsamples randomly selected; Mat\'ern $5/2$ covariance function), but with different learning and test sets for each replication. The covariance parameters are reestimated for each learning set, by minimizing the sum of log likelihoods for the SPV, PoE, GPoE1, GPoE2, BCM and RBCM aggregation techniques, and with the proposed leave-one-out estimation procedure for our nested aggregation method. The boxplots of the corresponding 20 mean square errors and mean negative log probability are reported in Figures~\ref{fig:box:plots} and \ref{fig:box:plots:p90}. These replications confirm the results obtained previously on single instances of the learning and test set: the proposed nested aggregation and covariance parameter estimation jointly give better prediction both for the predicted mean and variance than current existing aggregation techniques.

\begin{figure}[htp]
\centering
\begin{tabular}{ccc}
\includegraphics[height=5cm]{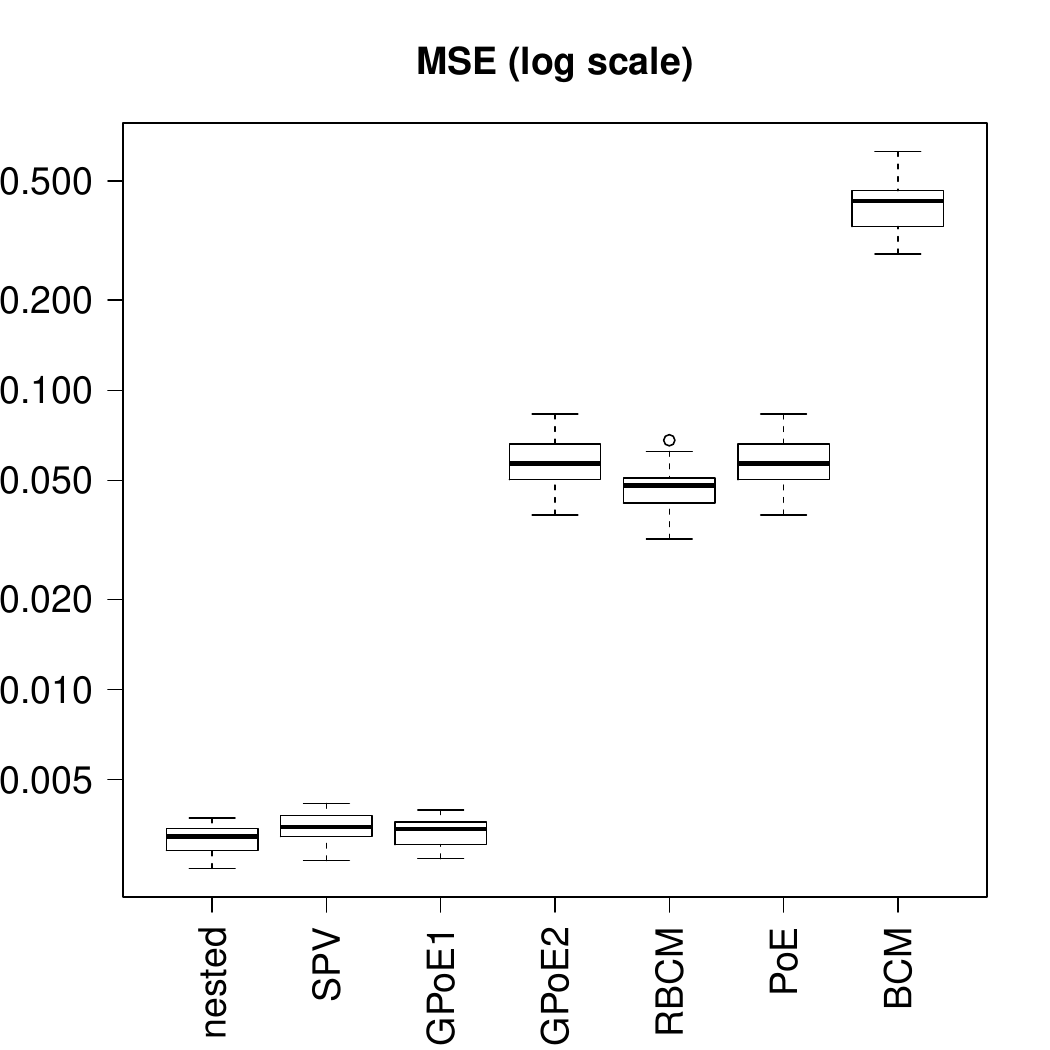} &
\includegraphics[height=5cm]{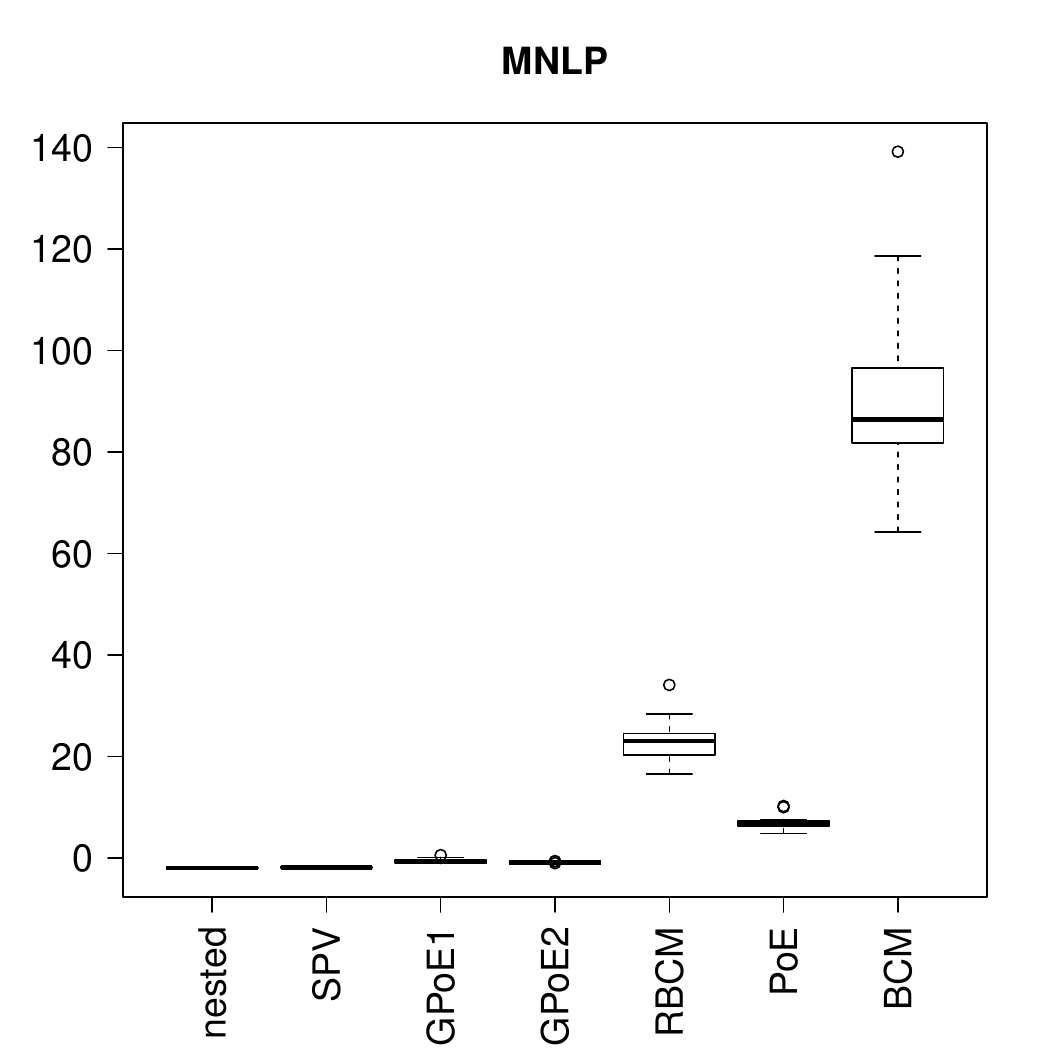} &
\includegraphics[height=5cm]{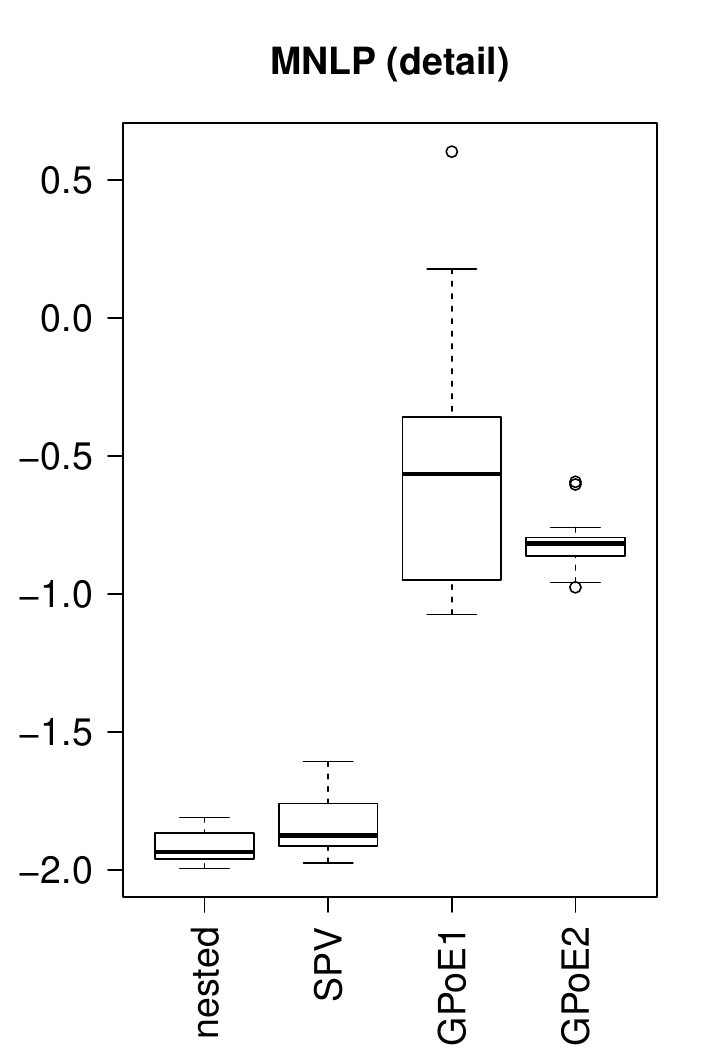}
\end{tabular}
\caption{Boxplots of $20$ values of the mean square error (MSE) prediction criterion and of the logarithm of the mean negative log probability (MNLP) prediction criterion where the learning and test sets are randomly generated. The settings are as in Table \ref{table:p20:kmeans:matern5:2} ($p=20$ subsamples obtained from the $k$-means algorithm; Mat\'ern $5/2$ covariance function). The covariance parameters are estimated by minimizing the sum of log likelihoods for the SPV, PoE, GPoE1, GPoE2, BCM and RBCM aggregation techniques, and with our proposed leave-one-out estimation procedure for the nested aggregation procedure.}
\label{fig:box:plots}
\end{figure}

\begin{figure}[htp]
\centering
\begin{tabular}{ccc}
	\includegraphics[height=5cm]{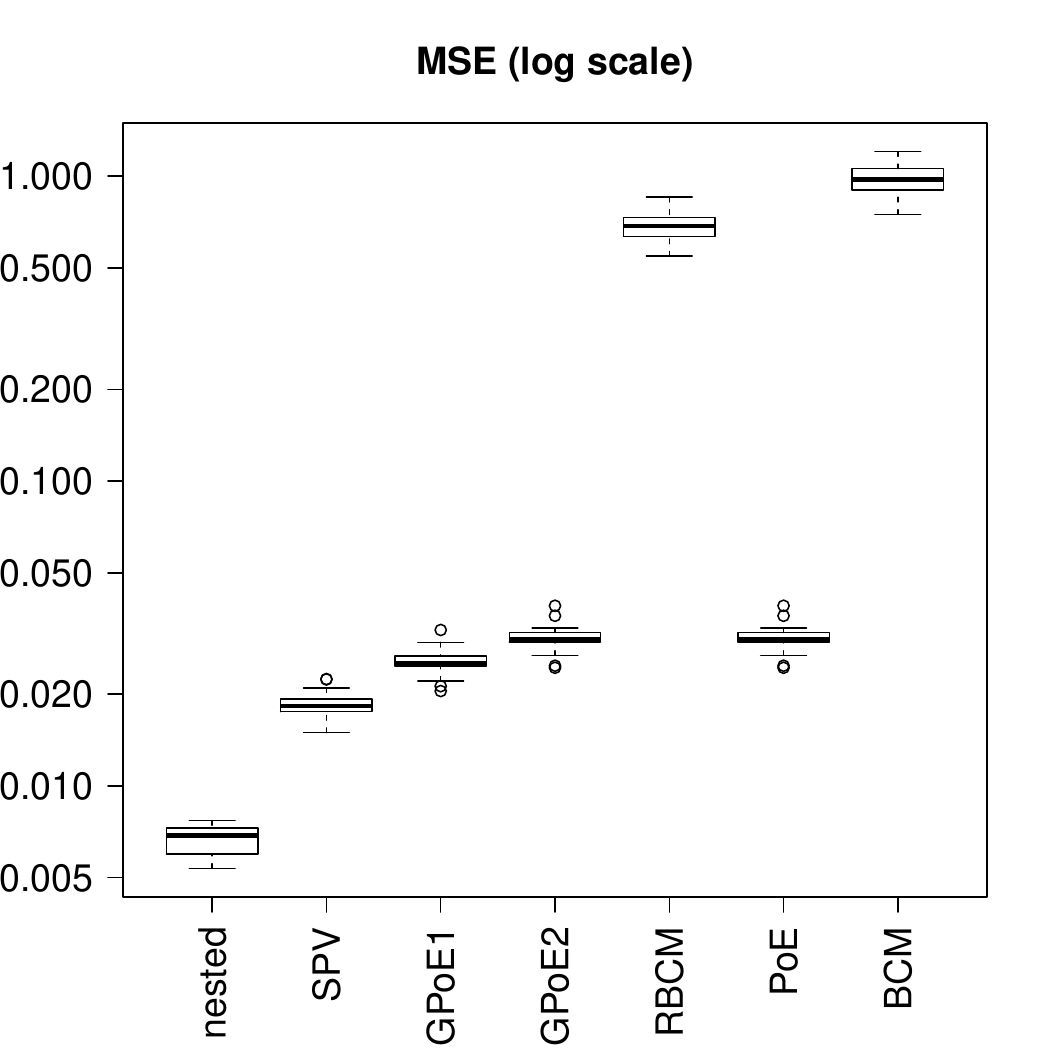} &
	\includegraphics[height=5cm]{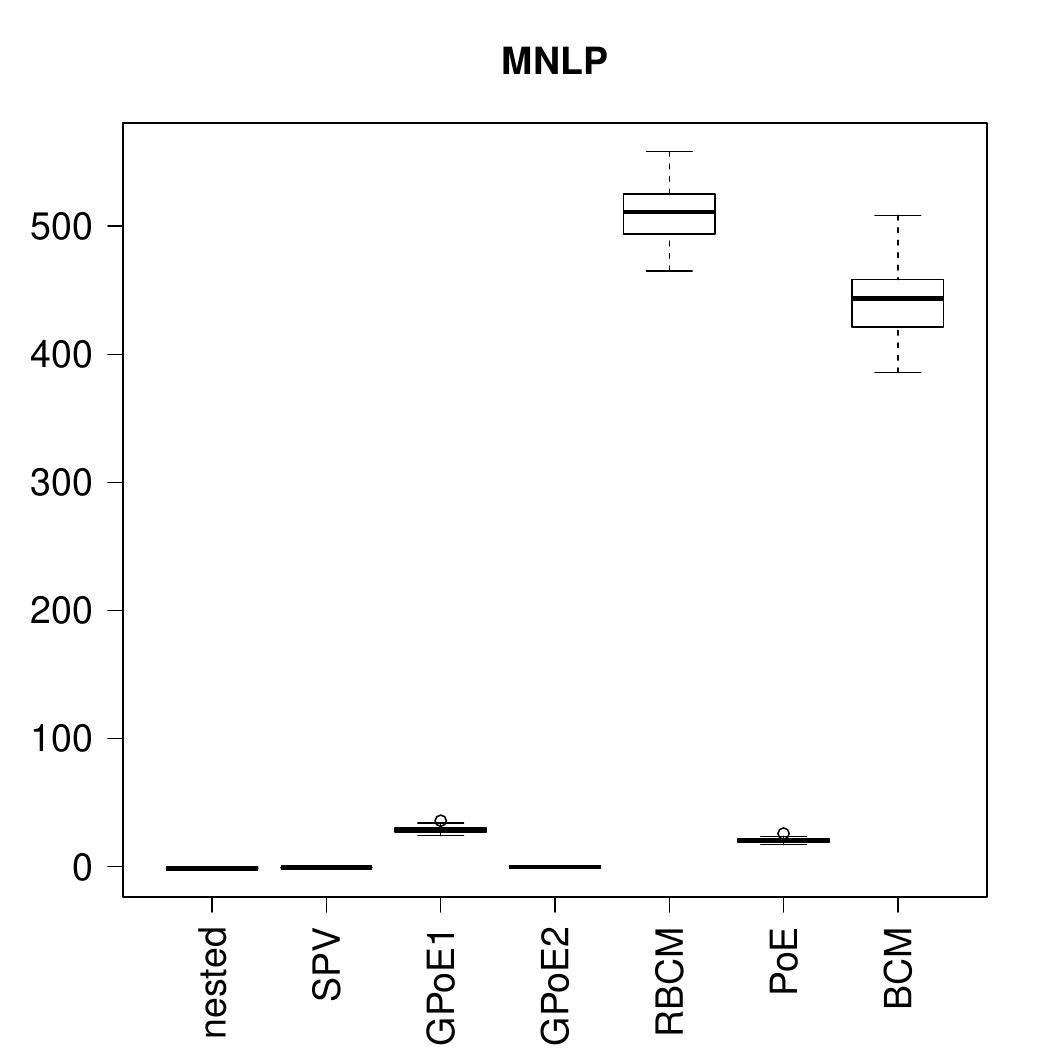} &
	\includegraphics[height=5cm]{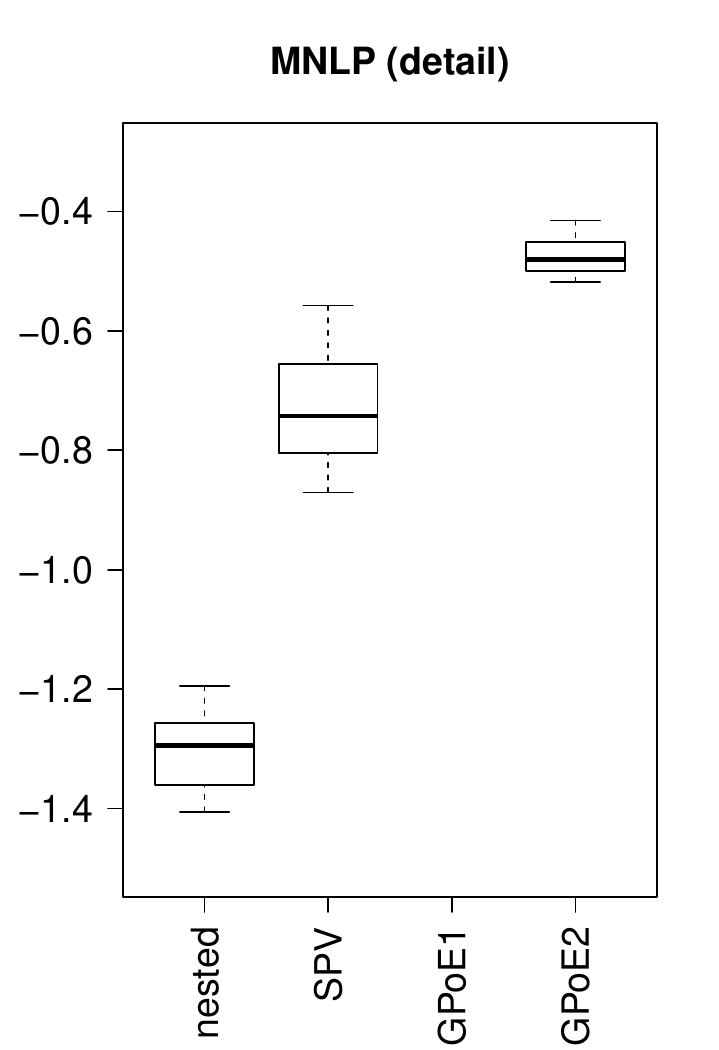}
\end{tabular}
\caption{Same settings as in Figure \ref{fig:box:plots} but with $p=90$ and where the subsamples are randomly selected.}
\label{fig:box:plots:p90}
\end{figure}

Of course, the improvement brought by our proposed aggregation scheme comes with a higher computational cost: the proposed estimation procedure takes a few hours on a personal computer, against a few tens of minutes for the minimization of the sum of the log likelihoods. Similarly, performing $1000$ predictions takes around $30$ seconds with our proposed optimal aggregation, against around $1$ second for the other simpler aggregation procedures. Nevertheless, we believe that the increased accuracy and robustness of the method we propose is worth the additional computational burden in many situations.

\FloatBarrier


\section{Conclusion}
\label{sec:6conclusion}

\paragraph{}
We have proposed a new method for aggregating sub-models based on subsets of observation points, with a particular emphasis on  Kriging  sub-models. Our method can be seen as an optimal linear weighting of sub-models, where the obtained weights are taking into account all pairwise covariances between the sub-models, thus avoiding some usual independence assumptions.

%

\paragraph{}
Compared to current existing aggregation techniques, we find several benefits to our our aggregation procedure.
First, it has some good theoretical
properties, like consistency or optimality based on a
slightly different process which can be simulated.
We refer again to \cite{bachoc2017} for details.
Second, a dedicated covariance parameter estimation procedure is provided, based on a gradient descent minimization of leave-one-out cross validation errors, where the predictions are performed using the proposed nested aggregation. Some
user-friendly code for computing both prediction and covariance parameter estimation is publicly available.

\paragraph{}

At last, numerical results are encouraging. In both simulated data and industrial application, our method is shown to outperform state-of-the-art aggregation techniques. This improvement comes with an increased computational cost compared to more basic aggregation methods, but the proposed nested aggregation remains applicable up to $n=10^6$ observation points, while exact Kriging inference becomes intractable around $n=10\,000$.

\paragraph{}

We would like to mention two avenues for future research. First, we show that the aggregation method we propose can be applied recursively, yielding a nested aggregation technique with smaller computational cost. It would be interesting to quantify the practical gain one could obtain on real data sets from this recursive aggregation. Second, we find that the stochastic gradient algorithm we propose could be further investigated. In particular, theoretical properties could be derived, the practical implementation could be improved, and the principle could be extended to other criteria for covariance parameter estimation.


\did{
 \section*{Acknowledgements}
Part of this research was conducted within the frame of the Chair in Applied Mathematics OQUAIDO, gathering partners in technological research (BRGM, CEA, IFPEN, IRSN, Safran, Storengy) and academia (Ecole Centrale de Lyon, Mines Saint-Etienne, University of Grenoble, University of Nice, University of Toulouse and CNRS) around advanced methods for Computer Experiments. The authors would like to warmly thank Dr. G\'eraud Blatman and EDF R\&D for providing us the industrial test case. They also thank both editor and reviewers for very precise and constructive comments on this paper. This paper has been finished during a stay of D. Rullière at Vietnam Institute for Advanced Study in Mathematics, the latter author thanks the VIASM institute and DAMI research chair (Data Analytics \& Models for Insurance) for their support.
}

\bibliographystyle{apalike}
\bibliography{biblio}


\appendix

\section{Proof of Proposition~\ref{prop:complexities}} \label{app:complexities}
\emph{Complexities}: under chosen assumption on $\alpha$ and $\beta$ coefficients, for a regular tree and in the case of simple Kriging sub-models, $\CalgoAlpha=\sum_{\nu=1}^{\numax} \sum_{i=1}^{n_\nu} \alpha c_\nu^3 =\alpha \sum_{\nu=1}^{\numax}  c_\nu^3 n_\nu$ and $\CalgoBeta=\sum_{\nu=1}^{\numax} \sum_{i=2}^{n_\nu} \sum_{j=1}^{i-1}\beta c^2_{\nu} =\frac{\beta}{2} \sum_{\nu=1}^{\numax}  n_\nu (n_{\nu}-1) c^2_{\nu}$. Notice that the sum starts from $\nu=1$ in order to include sub-models calculation.
\emph{Equilibrated trees complexities}: In a constant child number setting, when $c_\nu=c$ for all $\nu$, the tree structure ensures that $n_{\nu}=n/c^{\nu}$, thus as $c=n^{1/\numax}$, we get when $n \rightarrow +\infty$, $\CalgoAlpha \sim \alpha n^{1+\frac{2}{\numax}}$ and $\CalgoBeta \sim \frac{\beta}{2} n^2$. The result for equilibrated two-layer tree where $\numax=2$ directly derives from this one, and in this case $\CalgoAlpha \sim \alpha n^{2}$ and $\CalgoBeta \sim \frac{\beta}{2} n^2$ (it derives also from the expressions of $\CalgoAlpha$, $\CalgoBeta$, when $c_1=c_2=\sqrt{n}$, $n_1=\sqrt{n}$, $n_2=1$).
\emph{Optimal tree complexities}: One easily shows that under the chosen assumptions $\CalgoBeta \sim \frac{\beta}{2}n^2$. Thus, it is indeed not possible to reduce the whole complexity to orders lower than $O(n^2)$. However, one can choose the tree structure in order to reduce the complexity $\CalgoAlpha$. For a regular tree, $n_\nu=n/(c_1 \cdots c_{\nu})$ such that $\frac{\partial}{\partial c_k} n_{\nu} = -\Indic{\nu \ge k} n_\nu/c_k$. Using a Lagrange multiplier $\ell$, one defines $\xi(k)=c_k \frac{\partial}{\partial c_k} \left( \CalgoAlpha - \ell(c_1 \cdots c_{\numax} -n) \right) = 3\alpha c_k^3 n_k - \alpha \sum_{\nu=k}^{\numax}c_\nu^3 n_\nu - \ell c_1 \cdots c_{\numax}$.  The tree structure that minimizes $\CalgoAlpha$ is such that for all $k<\numax$, $\xi(k)=\xi(k+1)=0$. Using $c_{k+1} n_{k+1}=n_k$, one gets $3c_{k+1}^2=2 c_{k}^3$ for all $k<\numax$, and setting $c_1\cdots c_{\numax}=n$, $c_\nu = \delta \left( \delta^{-\numax} n\right)^{\frac{\delta^{\nu-1}}{2(\delta^{\numax}-1)}}$, $\nu=1, \ldots, \numax$,
with $\delta=\frac{3}{2}$. Setting $\gamma=\frac{27}{4}\delta^{-\frac{\numax}{\delta^{\numax}-1}}\left( 1-\delta^{-\numax}\right)$. After some direct calculations this tree structure corresponds to complexities, $\CalgoAlpha = \gamma \alpha  n^{1+\frac{1}{\delta^{\numax}-1}}$ and $\CalgoBeta \sim \frac{\beta}{2}n^2$.
In a two-layers setting one gets $c_1=\left(\frac{3}{2}\right)^{1/5} n^{2/5}$ and $c_2=\left(\frac{3}{2}\right)^{-1/5} n^{3/5}$, which leads to
$\CalgoAlpha = \gamma \alpha n^{9/5}$  and $\CalgoBeta= \frac{\beta}{2} n^2  - \frac{\beta}{2} \left(\frac{3}{2}\right)^{\frac{1}{5}}n^{\frac{7}{5}}$,
where $\gamma=(\frac{2}{3})^{-2/5}+(\frac{2}{3})^{3/5}\simeq 1.96$ (eventually notice that even for values of $n$ of order $10^5$, terms of order like $n^{9/5}$ are not necessarily negligible compared to those of order $n^2$, and that $\CalgoBeta$ is slightly affected by the choice of the tree structure, but the global complexity benefits from the optimization of $\CalgoAlpha$).\\
\emph{Storage footprint}: First, covariances can be stored in triangular matrices. So temporary objects $M$, $k$ and $K$ in Algorithm~\ref{algoKrigIteratif2} require the storage of $c_{\max}(c_{\max}+5)/2$ real values. For a given step $\nu$, $\nu \ge 2$, building all vectors $\alpha_i$ requires the storage of $\sum_{i=1}^{n_{\nu}} c_i^{\nu}=n_{\nu-1}$ values. At last, for a given step $\nu$, we simultaneously need objects $M_{\nu-1}, K_{\nu-1}, M_{\nu}, K_{\nu}$, which require the storage of $n_{\nu-1}(n_{\nu-1}+3)/2 + n_{\nu}(n_{\nu}+3)/2$ real values.
In a regular tree, as $n_\nu$ is decreasing in $\nu$, the storage footprint is $\mathcal{S} = (c_{\max}(c_{\max}+5) + n_1(n_1+5) + n_2(n_2+3))/2$. Hence the equivalents for $\mathcal{S}$ for the  different tree structures, $\mathcal{S}\sim n$ for the two-layer equilibrated tree, $\mathcal{S}\sim \frac{1}{2}n^{2-2/\numax}$ for the $\numax$-layer, $\numax>2$, and the indicated result for the optimal tree. Simple orders are given in the proposition, which avoids separating the case $\numax=2$ and a cumbersome constant for the optimal tree.

\end{document}